\documentclass{article}
\usepackage{graphicx}
\usepackage{wrapfig}
\usepackage{caption}
\usepackage[table]{xcolor}
\PassOptionsToPackage{numbers, compress}{natbib}
 \usepackage[preprint]{neurips_2026}

\usepackage[utf8]{inputenc} 
\usepackage[T1]{fontenc}    
\usepackage{hyperref}       

\usepackage{url}            
\usepackage{booktabs}       
\usepackage{amsfonts}       
\usepackage{nicefrac}       
\usepackage{microtype}      
\usepackage{xcolor}         
\usepackage{enumitem}
\usepackage{amsmath}
\usepackage{amssymb}
\usepackage{amsthm}
\usepackage[capitalize,noabbrev]{cleveref}
\usepackage{algorithm}
\usepackage[noend]{algorithmic}
\usepackage{multirow} 
\usepackage{tabularx}
\usepackage{longtable}

\newtheorem{definition}{Definition}[section]
\newtheorem{assumption}{Assumption}[section]
\newtheorem{lemma}{Lemma}[section]
\newtheorem{theorem}{Theorem}[section]
\newtheorem{proposition}{Proposition}[section]
\newtheorem{corollary}{Corollary}[section]
\newtheorem{remark}{Remark}[section]

\newcommand{\TV}{\mathrm{TV}}
\newcommand{\Law}{\mathcal{L}}

\newcommand{\NFE}{\mathrm{NFE}}
\newcommand{\dd}{\mathrm{d}}
\makeatletter
\DeclareRobustCommand{\cev}[1]{%
  {\mathpalette\do@cev{#1}}%
}
\newcommand{\do@cev}[2]{%
  \vbox{\offinterlineskip
    \sbox\z@{$\m@th#1 x$}%
    \ialign{##\cr
\hidewidth\reflectbox{$\m@th#1\vec{}\mkern4mu$}\hidewidth\cr
      \noalign{\kern-\ht\z@}
      $\m@th#1#2$\cr
    }%
  }%
}
\makeatother

\title{Accelerating Discrete Diffusion Models with Parallel-In-Time Sampling}

\author{
  Yu Yao \\
  The University of Tokyo\\
  \texttt{yuyao@g.ecc.u-tokyo.ac.jp} \\
  \And
  Huanjian Zhou\thanks{Corresponding to: Huanjian Zhou} \\
  The University of Tokyo \\
  \texttt{zhou-huanjian185@g.ecc.u-tokyo.ac.jp} \\
  \And
  Andi Han \\
  The University of Sydney \\
  \texttt{andi.han@sydney.edu.au} \\
  \And
  Wei Huang \\
  RIKEN AIP \& The Institute of Statistical Mathematics \\
  \texttt{wei.huang.vr@riken.jp} \\
  \And
  Masashi Sugiyama \\
  RIKEN AIP \& The University of Tokyo \\
  \texttt{sugi@k.u-tokyo.ac.jp} \\
}

\usepackage{amsmath,amsfonts,bm}

\def\eqref#1{equation~\ref{#1}}

\def\1{\bm{1}}

\def\vp{{\bm{p}}}

\def\vs{{\bm{s}}}

\def\vx{{\bm{x}}}
\def\vy{{\bm{y}}}

\def\mQ{{\bm{Q}}}

\DeclareMathAlphabet{\mathsfit}{\encodingdefault}{\sfdefault}{m}{sl}
\SetMathAlphabet{\mathsfit}{bold}{\encodingdefault}{\sfdefault}{bx}{n}

\def\sD{{\mathbb{D}}}

\begin{document}

\maketitle

\begin{abstract}
Discrete diffusion models are widely used for learning and generating discrete distributions. As the generation process is inherently sequential, the acceleration of sampling is of significant importance.
In this work, we parallelize the mainstream $\tau$-leaping algorithm for absorbing discrete diffusion in a Continuous-Time Markov Chain (CTMC) framework. By leveraging the continuous-time stochastic integral form of the $\tau$-leaping algorithm and the Picard iteration method, we achieve parallel-in-time sampling acceleration and provide a proof of exponential-factorial convergence for our algorithm. We improve the overall time complexity of $\tau$-leaping under absorbing settings from ${\mathcal{O}}(d \log S)$ to ${\mathcal{O}}(\log (d\log S)\cdot \log d)$ with respect to NFE. Empirically, our method shows consistent acceleration across synthetic and real-data settings. The new sampler achieves at most $7$--$9\times$ runtime speedup for synthetic distribution, and maintains the same quality with $50\%$ fewer NFE and $1.45$--$1.86\times$ runtime speedups in image/text tasks on a single GPU. Our research expands the potential of discrete diffusion models for efficient parallel inference, with broader implications for applications such as molecular structure and language generation.
\end{abstract}

\section{Introduction}
Diffusion models over discrete spaces~\citep{austin2021structured} have become central to generative modeling for categorical data, with increasing impact across applications including molecular design~\citep{vignac2022digres}, protein engineering~\citep{gruver2023protein}, DNA sequence generation under biophysical constraints~\citep{sarkar2024designing}, and high-fidelity generation of text~\citep{zheng2023reparameterized}, music~\citep{yang2023diffsound}, and images~\citep{lezama2022discrete}. One of the mainstream discrete diffusion is the Absorbing discrete diffusion, which progressively replaces discrete tokens with a special mask or absorbing state, then learns to reconstruct the original data by reversing this corruption process. However, the algorithms underlying the above results are highly sequential and fail to fully exploit contemporary parallel computing resources for modern GPUs. While parallel-in-time algorithms for continuous diffusion models have been extensively studied~\citep{anari2024fast,chen2024accelerating, shih2023parallel, zhouparallel}, their counterparts for discrete diffusion models remain largely unexplored. This gap motivates our investigation into the question:
\begin{center}
\textit{Whether parallelization can fundamentally accelerate sampling for discrete diffusion models}?
\end{center}

\subsection{Our contributions}
In this paper, we make significant progress in both theoretical and practical sides.
The {key contributions} of this work are summarized as follows:
\begin{itemize}[leftmargin =2em]
\item  We introduce the \emph{first} \emph{parallel-in-time} algorithm for $\tau$-leaping methods \citep{campbell2022continuous,gillespie2001approximate} for absorbing discrete diffusion inference, namely the Picard $\tau$-leaping method~(Algorithm~\ref{predictor}).

\item  We provide well-established theoretical analysis of the method, proving ${\mathcal{O}}(\log (d\log S)\cdot \log d)$ time complexity with respect to NFE compared with ${\mathcal{O}}(d \log S)$ complexity of mainstream sequential methods with extra $\widetilde{\mathcal{O}}(dS)$ space complexity\footnote{We note, in this paper, that the space complexity refers to the state-level space after each iteration based the number of words~\citep{cohen2023streaming} instead of the number of bits~\citep{goldreich2008computational} to denote the
approximate required storage.} cost~(Theorem~\ref{thm:global-complexity-liang_main}). We summarize the comparison between existing methods and our results in Table \ref{table:results}.

\item  We verify our method on both synthetic and real-world tasks and achieve substantial improvements in sampling efficiency over various sequential and acceleration methods~(Section~\ref{sec:exp}). 
\end{itemize}

\begin{table}[t]
\footnotesize
\renewcommand\arraystretch{1.4}
\centering
\captionsetup{skip=8pt}
\caption{Comparisons of our parallel and sequential methods for discrete diffusion sampling given $\delta$-accuracy of the score. Here, $S$ denotes the size of vocabulary.}

\label{table:results} 

\begin{tabular}{|c|c|c|c|} 
    \hline
        Method       & Time Complexity & Space Complexity & Metrics\\
    \hline
    $\tau$-leaping \citep[Theorem 4.7]{rendiscrete} & $\mathcal{O}(\frac{d^2\log^2(d/\delta^2)}{\delta^2})$ & $\mathcal{O}(d)$ & KL-divergence\\
    \hline
    $\tau$-leaping \citep[Theorem 3.1.3]{conforti2025non} & $\mathcal{O}(\frac{dS}{\delta^2})$& $\mathcal{O}(d)$ & Total Variation \\
    \hline 
    $\tau$-leaping \citep[Theorem 2]{liangabsorb} & $\mathcal{O}(\frac{d S}{\delta})$& $\mathcal{O}(d)$ & KL-divergence\\
    \hline
    $\tau$-bridging \citep[Theorem 2]{dmitriev2026efficient} & $\mathcal{O}(\frac{\mathcal{D}(q_0)}{\delta})\leq\mathcal{O}(\frac{d\log S}{\delta})$& $\mathcal{O}(d)$ & KL-divergence\\
    \hline
    First-Hitting Sampler \citep[Theorem 4]{liang2026sharp} & Exact $d$ & $\mathcal{O}(d)$ & KL-divergence\\
    \hline
    \rowcolor{blue!15}
    Picard $\tau$-leaping [Ours, Theorem \ref{thm:global-complexity-liang_main}] & $\mathcal{O}\left(
\log\frac{d \log S}{\delta^2}
\cdot
\log\frac{d}{\delta}
\right)$ &  $ \widetilde{\mathcal{O}}(d^{2}S)$ & Total Variation\\
    \hline
\end{tabular}
\end{table}

\subsection{Technical Overview: From Coordinate Parallelism to Time Parallelism}
 Sampling algorithms for absorbing discrete diffusion can be broadly divided into \emph{Approximate Time-Discretization Methods} and \emph{Exact Simulation Methods}. The typical approximate sampler, $\tau$-leaping sampler, freezes the reverse rates over short time intervals and applies multiple coordinate-level jumps \emph{in parallel} within each step. A general stochastic-integral framework for discrete diffusion was developed by \citep{rendiscrete}, who represent discrete diffusion through Poisson random measures with evolving intensity
 but its generic $\tau$-leaping bounds are conservative for absorbing processes.

Many recent work focuses specifically on the absorbing dynamics. \citep{huang2025complexity} study the complexity of masked discrete diffusion and propose Mask-Aware Truncated Uniformization (MATU), which leverages the fact that masked tokens cannot be unmasked multiple times and adapts the truncation of outgoing rates to the number of remaining masked coordinates. \citep{liangabsorb} gives the first rigorous convergence guarantees for absorbing $\tau$-leaping under bounded score estimates. \citep{conforti2025non} provide sharper non-asymptotic bounds for masked and random-walk discrete diffusion, using the evolution and monotonicity of discrete scores to avoid strong boundedness assumptions. \citep{dmitriev2026efficient} analyze discrete diffusion through information-theoretic quantities, proving sharp upper bound that intrinsically depends on the structural properties of the target distribution. More specifically, sampling highly structural distribution with sparse dependency graph may allow more parallel-in-token unmask operations at each time step, which leads to lower time complexity.
Together, these results highlight that masking diffusion has exploitable structure beyond generic finite-state CTMCs, and show that sequential masking samplers can achieve at most linear complexity up to logarithmic factors.

Compared with previous sequential methods and analysis, our method keeps the same fine grid but approximate the sequential chain through parallelism along the time horizon. 
We group many fine intervals into a larger block and solve the whole block transition by Picard iteration. 
In each Picard round, all microsteps in the block are evaluated in parallel from the previous round's trajectory using the same block randomness generated at the beginning of the algorithm. 
Compared to the Picard methods for continuous diffusion models~\citep{chen2024accelerating, shih2023parallel, zhouparallel}, we additionally apply a first-hitting truncation inside the block for masking diffusion, preserving the absorbing structure while making the block update compatible with parallel prefix operations.
By proving an exponential-factorial contraction of the Picard error, we derive $O(\log d)$ level estimation for both number of blocks and Picard iterations, which eventually leads to $\mathrm{poly}\log(d)$ time complexity.

\subsection{Other Related Work}
\paragraph{Exact simulation of reverse process with learned score.}
Exact simulation methods, such as uniformization for uniform-rate chains~\citep{chen2024convergence} and the First-Hitting Sampler (FHS)~\citep{liang2026sharp,zheng2024masked} for absorbing chains, simulate the reverse CTMC by separating the sampling of jump times from the sampling of jump destinations. 
This removes time-discretization error and, in the absorbing case, can directly exploit the fact that each coordinate is unmasked at most once. 
In the ideal exact-score setting, such methods provide unbiased simulation of the learned reverse  with $d$ steps sampling \citep{liang2026sharp}. 
However, exact simulation does not necessarily lead to better generation quality in practice. Recent empirical results even show that exact simulation can underperform approximate solvers despite having no discretization error \citep{ren2025fast}. One possible reason is that exact samplers concentrate many jumps near the terminal phase of the reverse process, precisely where score estimation is most singular and inaccurate. 
This motivates continued study of approximate samplers, especially the $\tau$-leaping method, whose discretization can act as a stabilizing numerical approximation of learned-score errors.

\paragraph{Other acceleration strategies}
A large body of empirical work accelerates discrete diffusion by changing the decoding strategy rather than the time integration scheme. 
Parallel decoding methods unmask multiple tokens per iteration using confidence, entropy, margin, or related criteria, thereby exploiting spatial or token-level parallelism \citep{chang2022maskgit, ringel2026dependency, xie2025variational, wu2025fast, yu2025dimple}. 
These methods are effective in practice, but the sampling horizon still consists of a sequence of denoising rounds. 
Other approaches modify the generation order or introduce auxiliary planning modules. 
DDPD \citep{liu2024think} separates generation into a planner and a denoiser, allowing the model to decide which corrupted positions should be refined next; related path-planning methods study how the unmasking order affects sample quality \citep{peng2025path}. High-order solvers \citep{ren2025fast} extend ideas such as trapezoidal integration to jump processes, reducing discretization error and permitting larger steps under the same computation budget .

\section{Preliminaries on Discrete Diffusion Models}
\label{sec:preliminaries}
In this section, we briefly introduce the background knowledge about the mathematical framework of discrete diffusion and its properties under mask settings. Throughout the sampling and theoretical analysis, \(t\) denotes the forward noising time where \(t=0\) is the clean-data endpoint, and the early-stopped reverse sampler runs from \(T\) down to \(\eta>0\).

\subsection{Continuous-Time Markov Chains and Poisson Integrals}

In discrete diffusion models, the forward process is a continuous-time Markov chain (CTMC) $(\vx_t)_{t\in[0,T]}$ on a finite state space $\mathcal{X}$. Let $\vp_t\in\Delta^{|\mathcal{X}|}$ denote the law of $\vx_t$ as a column vector. The forward equation is
$  \frac{\mathrm{d}\vp_t}{\mathrm{d}t}=\mQ_t\,\vp_t,
  \mQ_t=\bigl(\mQ_t(y,x)\bigr)_{x,y\in\mathcal{X}},
  $
where $\mQ_t$ is a rate matrix with (i) $\mQ_t(x,y)\ge 0$ for $x\neq y$ and (ii) $\mQ_t(x,x)=-\sum_{y\neq x}\mQ_t(y,x)$. We write $\widetilde \mQ_t:=\mQ_t-\mathrm{diag}(\mQ_t)$ for the off-diagonal part. 
The time-reversed (backward) process $\bigl(\cev{\vx}_t\bigr)_{t\in[0,T]}$, with $\cev{*}_t:=*_{T-t}$, is again a CTMC with law $\cev{\vp}_s$ and generator $\cev{\mQ}_t$ satisfying~\citep{kelly2011reversibility}
$  \frac{\mathrm{d}\,\cev{\vp}_t}{\mathrm{d}t}=\cev{\mQ}_t\,\cev{\vp}_t,$
where for $x\neq y$,
$  \cev \mQ_t(y,x)=\frac{\cev \vp_t(y)}{\cev \vp_t(x)}\,\cev \mQ_t(x,y),$ and $
  \cev \mQ_t(x,x)=-\sum_{y'\neq x}\cev \mQ_t(y',x),$
and $\cev{\mQ}_t:=\mQ_{T-t}$.

According to Proposition 3.2 in~\citep{rendiscrete}, discrete diffusion models can also be interpreted as stochastic integrals with Poisson random measure.
%
%
The forward process in discrete diffusion models can thus be represented by the following stochastic integral:
$	\vx_t = \vx_0 + \int_0^t \int_{\sD} \nu  N[\lambda](\mathrm{d} t, \mathrm{d} \nu),$
where the intensity $\lambda$ is defined as
$
	\lambda_t(\nu, \omega) = \cev \mQ_t(x_{t^-}(\omega) + \nu, x_{t^-}(\omega))
$
if $x_{t^-}(\omega) + \nu \in \mathcal{X}$ and 0 otherwise. Here, the outcome $\omega \in \Omega$ and $x_{t^-}$ denotes the left limit of the c\`adl\`ag process $x_t$ at time $t$ with $x_{0^-} = x_0$. We will also omit the variable $\omega$, should it be clear from context.

The backward process in discrete diffusion models can also be represented similarly as
\begin{equation}
	\vy_t = \vy_0 + \int_0^t \int_{\sD} \nu  N[\mu](\mathrm{d} s, \mathrm{d} \nu),
	\label{eq:backward_integral_nu}
\end{equation}
where the intensity $\mu$ is defined as $\mu_t(\nu, \omega) = \cev{s}_t(\vy_{t^-}, \vy_{t^-} + \nu) \cev \mQ_s(\vy_{s^-}, \vy_{s^-} + \nu)$. During inference, $\widehat \mu$ defined by replacing the true score $\vs_t$ with the neural network estimated score $\widehat \vs_t$ is used.

\subsection{Masking Diffusion}
\label{sec:masking-diffusion}

In this work, we focus on absorbing, or masking discrete diffusion on the clean state space
\(\mathcal X_0=[S]^d\) and the masked state space
\(\mathcal X_M=([S]\cup\{\mathrm{MASK}\})^d\).  The forward process independently replaces each clean coordinate by \(\mathrm{MASK}\) and keeps \(\mathrm{MASK}\) absorbing.  Let \(\beta_t\ge0\) be the forward masking rate (i.e. noise schedule) and
$\alpha_t
    :=
    \exp\!\left(-\int_0^t \beta_s\,ds\right)$.
Then, for a clean sample \(x_0\in[S]^d\), the forward marginal is $q_t(z\mid x_0)
    =
    \prod\nolimits_{i=1}^d
    \Big[
    \alpha_t\,\mathbf 1\{z_i=x_{0,i}\}
    +
    (1-\alpha_t)\,\mathbf 1\{z_i=\mathrm{MASK}\}
    \Big], z\in\mathcal X_M .$
Equivalently, the forward CTMC generator only allows transitions from ordinary tokens to \(\mathrm{MASK}\):
\[
    Q_t(z,z^{i\to \mathrm{MASK}})
    =
    \beta_t\,\mathbf 1\{z_i\neq \mathrm{MASK}\},
    \qquad
    Q_t(z,z)
    =
    -\beta_t\,|\{i:z_i\neq \mathrm{MASK}\}|.
\]
Here \(z^{i\to \mathrm{MASK}}\) denotes the state obtained by replacing coordinate \(i\) of \(z\) with \(\mathrm{MASK}\).

The reverse process has a particularly simple structure.  Let $O(z):=\{i:z_i\neq \mathrm{MASK}\}$
be the observed coordinates of a partially masked state \(z\).  For a masked coordinate \(i\notin O(z)\), define the clean-data posterior $\pi_i(c\mid z)
    :=
    \mathbb P_{X_0\sim p_{\mathrm{data}}}
    \!\left(
    X_{0,i}=c
    \,\middle|\,
    X_{0,O(z)}=z_{O(z)}
    \right), c\in[S].$
Because the masking channel treats all hidden token values identically, this posterior is independent of the corruption time \(t\).  The reverse rate from \(z\) to the state \(z^{i\leftarrow c}\), obtained by replacing \(\mathrm{MASK}\) at coordinate \(i\) with token \(c\), is
\[
    \mu_t(z^{i\leftarrow c},z)
    =
    \beta_t\frac{q_t(z^{i\leftarrow c})}{q_t(z)}
    =
    \frac{\beta_t\alpha_t}{1-\alpha_t}\,\pi_i(c\mid z)
    =
    a_t\,\pi_i(c\mid z),
    \qquad
    a_t:=\frac{\beta_t\alpha_t}{1-\alpha_t}.
\]
Thus the reverse rate separates into a scalar time factor \(a_t\) and a time-independent conditional token posterior.  A learned score model is used to approximate this posterior, giving rates of the form $\widehat\mu_t(z^{i\leftarrow c},z)=a_t\,\widehat p^\theta_i(c\mid z)$ and $
    \sum\nolimits_{c\in[S]}\widehat p^\theta_i(c\mid z)=1$.
Finally, masking diffusion has a first-hitting structure.  Along reverse sampling, a coordinate remains masked until its first reveal time \(\tau_i\in[\eta,T]\), and then stays fixed as $X_i(t)=\mathrm{MASK}$ for $t>\tau_i$ and $X_i(t)=c_i$ for $t\le \tau_i$.
This absorbing first-hitting property distinguishes masking diffusion from uniform or random-walk discrete diffusion, where coordinates may jump repeatedly, and is the structural property exploited by our Picard sampler.

\section{Parallel Sampling for Absorbing Discrete Diffusion Models: Algorithm~\ref{predictor}}
\label{sec:parallel-sampling}
In this section, we introduce the parallel-in-time $\tau$-leaping algorithm for absorbing discrete diffusion models. The key idea is to extend the Picard iteration from continuous diffusion models to the stochastic integral form of $\tau$-leaping in discrete state spaces.

\begin{algorithm}[ht]
\caption{Parallel $\tau$-Leaping Algorithm for Discrete Diffusion Model Sampling}
\label{predictor}
\begin{algorithmic}[1]
\REQUIRE $\widehat y_{t_0}\sim q_T$, large time grid $(t_n)_{n\in[0,N]}$ with $t_0=T$, $t_N=\eta$, and $t_{n+1}<t_n$; small time grid $(\tau_{n,m})_{n\in[0,N-1],m\in[0,M]}$ with $\tau_{n,0}=t_n$, $\tau_{n,M}=t_{n+1}$ and positive step size $\epsilon=t_n-t_{n+1}$ divided by $M$, i.e. $\tau_{n,m}=t_n-m\epsilon$;
Picard depth $K_p$; intensity $\widehat\mu_s^\theta$; pre-sampled random seeds $(\xi_n)_{n\in[0,N-1]}$; First-hitting truncation operation $\mathrm{FHT}(\cdot)$ .
\ENSURE An early-stopped sample $\widehat y_{t_N}$ with $t_N=\eta$.

\FOR{$n=0$ to $N-1$}
    \STATE \textbf{Initial Guess:} $\widehat{y}^{(0)}_{\tau_{n,m}}\gets\widehat{y}_{t_n} $ \text{for all} $m\in[0,M]$\\
    \FOR{$k=0$ to $K_p-1$}
        \FOR{$j=0$ to $M-1$ \textbf{in parallel}}
            \STATE $\widehat{\mu}^{\theta}_{j}\gets \widehat{\mu}^{\theta}_{\tau_{n,j}}\big( \cdot|\widehat{y}^{(k)}_{\tau_{n,j}}\big)$
            \STATE $J_{j,y'}\sim \mathcal{P}(\widehat{\mu}^{\theta}_{j}(y')\cdot\epsilon;\xi_n)$ \text{for all} $y'\in\mathcal{X}_{neighbor}$
            \STATE $\Delta\widehat{y}^{(k)}_j=\sum_{y'\in\mathcal{X}_{neighbor}}\big(y'-\widehat{y}^{(k)}_{\tau_{n,j}}\big)\cdot J_{j,y'}$
        \ENDFOR
        \STATE \textbf{First-Hitting Truncation:} $(\Delta\widehat y^{(k)}_0,\ldots,\Delta\widehat y^{(k)}_{M-1})\gets\mathrm{FHT}\!\left(\Delta\widehat y^{(k)}_0,\ldots,\Delta\widehat y^{(k)}_{M-1}\right)$
        \FOR{$m=1$ to $M$ \textbf{in parallel}}
            \STATE $\widehat{y}^{(k+1)}_{\tau_{n,m}}\gets\widehat{y}_{t_n}+\sum^{m-1}_{j=0}\Delta\widehat{y}^{(k)}_{j}$
        \ENDFOR
    \ENDFOR
    \STATE $\widehat{y}_{t_{n+1}}\gets \widehat{y}^{(K_p)}_{\tau_{n,M}}$
\ENDFOR
\end{algorithmic}
\end{algorithm}

\paragraph{Discretization Scheme.}
Following the time convention above, the reverse sampler is executed on the forward-noise interval \([\eta,T]\) in decreasing time order. We use a large noise-time grid
$T=t_0>t_1>\cdots>t_N=\eta$.
For block \(n\), the interval \([t_{n+1},t_n]\) is divided into \(M\) fine cells using $\tau_{n,m}=t_n-m\epsilon_n$,
$\epsilon_n=\frac{t_n-t_{n+1}}{M},m=0,\ldots,M$.
The positive number \(\epsilon_n\) is the fine-step width used in the Poisson proposal mean, and can be set as $\epsilon_n=\epsilon$ under constant width schedule.

\paragraph{Picard update.} When describing the standard $\tau$-leaping algorithm with the form in \eqref{eq:backward_integral_nu}, the main update step in Picard iteration can be described as
\begin{equation}
\begin{split}
\label{eqn:picard_update}
\widehat{y}^{(k+1)}_{\tau_{n,m}}
&=
\widehat{y}_{t_n}
+
\sum\nolimits_{j=0}^{m-1}
\bigg(
\sum\nolimits_{y'\in\mathbb{X}}
\big(
y'-\widehat{y}^{(k)}_{\tau_{n,j}}
\big)
\cdot
\mathcal{P}
\big(
\widehat{\mu}^{\theta}_{\tau_{n,j}}
\big(
y'\mid \widehat{y}^{(k)}_{\tau_{n,j}};\xi_n
\big)
\cdot \epsilon_n
\big)
\bigg).
\end{split}
\end{equation}
where $\widehat{y}^{(k)}_{\tau_{n,j}}$ is the sample state a $t=t_n-j\epsilon$ in the $k$-th Picard iteration. $\mathcal{P}(\widehat{\mu}^{\theta}_{\tau_{n,j}}(y'|\widehat{y}^{(k)}_{\tau_{n,j}};\xi_n)$ denotes the number of jumps from $\widehat{y}^{(k)}_{\tau_{n,j}}$ to $y'$ during the small time interval $\epsilon$. In the block-wise parallel algorithm, the computation for each block starts from its initial state $\widehat{y}_{t_n}$ which is also the terminal value of the last block. Check Lines 4--7 in Algorithm~\ref{predictor} for details. The sampler admits a coordinate-token factorization \citep{campbell2022continuous}: each fine cell proposes token updates over \(d\) coordinates and \(S\) vocabulary entries, giving \(O(dS)\) local proposal cost rather than \(S^d\) dependence.  The Picard update then replaces the serial dependence on the previous fine step by dependence on the full trajectory from the previous iteration, and all proposals inside a block can be computed in parallel. After applying first-hitting truncation, the next block trajectory is reconstructed from the block start state using a parallel prefix sum over the truncated jump vectors. With pre-sampled shared random seeds $\xi_n$, the Poisson sampling performs as a deterministic map for convergence in each block. Since multi-time jump at the same time on one coordinate is not well-defined, such case is also truncated~\citep{liang2025discrete}.
\paragraph{First-hitting truncation.}
First-hitting truncation is designed for absorbing settings in Algorithm~\ref{alg:fht} (check Appendix~\ref{appd:FHS_alg} for details). Given the local proposal events in a block prefix before the trajectory reconstruction, the first-hitting operation returns the state obtained by applying, for each coordinate, only its earliest proposed token.

\section{Theoretical Guarantees}
\label{sec:theory}
In this section, we provide a proof of the Picard Convergence and an approximated error bound of our algorithm under specific definition and assumptions.



\subsection{Assumptions}
\label{subsec:assumptions}
We now introduce the definitions and assumptions used for the proof of theoretical guarantees. Definition~\ref{def:fine-cell} and~\ref{def:block-masses} provide necessary measure for evaluating the random proposal and error propagation. Assumption~\ref{ass:score-estimation-error} and~\ref{ass:bounded-score-estimate} are standard score controls inherited from the absorbing $\tau$-leaping convergence theory in \citep{liangabsorb} which allow us to isolate the additional error introduced by parallelizing the sequential sampler. Assumption~\ref{ass:event-switching} is a Dobrushin-style perturbation controlling condition, which can be considered as an analogue to the Lipschitz condition in continuous diffusion.

\begin{assumption}[\bf Score Estimation Error]
\label{ass:score-estimation-error}
Let \(t_0=T>t_1>\cdots>t_{N_{\mathrm{fine}}}=\eta\) be the serial fine grid. The estimated score satisfies
$    \sum_{\ell=0}^{N_{\mathrm{fine}}-1}
    (t_\ell-t_{\ell+1})
    \mathcal L_{\mathrm{SE}}(\widehat s_{t_\ell})
    \le
    \varepsilon^2_{\mathrm{score}}.$
\end{assumption}

\begin{assumption}[\bf Bounded Score Estimate]
\label{ass:bounded-score-estimate}
There exists $M_{\mathrm{score}}>0$ such that for all \(x,y\in[S]^d\) with \(Q_{t_\ell}(y,x)>0\),
$    |\log \widehat s_{t_\ell}(y,x)|
    \le \log M_{\mathrm{score}}.$
\end{assumption}

\begin{definition}[\bf Fine cells and local proposal events]
\label{def:fine-cell}
Let $\mathcal Q_d$ be the set of fine-grid cells covering the early-stopped interval $[\eta,T]$. For each $q\in\mathcal Q_d$, let $\Delta_q$ be its fine-step width, $t_q$ its representative time, and $a_q=a(t_q)$ be the total unmasking rate of one active masked coordinate. Define
\[
\lambda_q=\Delta_q a_q,
\qquad
\rho_q=\lambda_q e^{-\lambda_q}\le \lambda_q=\Delta_q a_q
\]
Here $\rho_q$ is the one-proposal probability at one active coordinate in cell $q$.

For an input state $z$ and cell randomness $\omega_q$, let
$C_q(z,\omega_q)\subseteq[d]\times[S]$
be the set of valid local proposal events generated in cell $q$. An event $(i,c)\in C_q(z,\omega_q)$ means that coordinate $i$ is locally proposed as token $c$.
\end{definition}

\begin{assumption}[\bf Normalized event switching]
\label{ass:event-switching}
There exists a constant $L_\star>0$ such that, for every fine cell $q$ and all relevant states $z,z'\in\mathcal X_M$,
\[
\mathbb E_{\omega_q}\left[\left|C_q(z,\omega_q)\triangle C_q(z',\omega_q)\right|\right]
\le
L_\star\rho_q d_H(z,z').
\]
\end{assumption}
This assumption works as a model-sensitivity condition in the Picard proof. The factor $\rho_q$ captures how likely cell $q$ is to produce a local proposal at all, and $L_\star$ controls how strongly a token-level perturbation of the input context can change the local proposal event set, whose constant property is also verified in Appendix~\ref{exp:assumption_test} and Figure~\ref{fig:L_star}.
\paragraph{Connection to score sensitivity.}
The event-switching condition can be viewed as a proposal-level analogue of a score Lipschitz condition. 
For masking diffusion, the reverse rate factorizes as $\widehat\mu_{\theta}(z)=a_q\widehat p_i^\theta(c\mid z)$.
If the normalized token posterior satisfies
\[
\sum\nolimits_i
\left\|
\mathbf 1\{z_i=\mathrm{MASK}\}\widehat p_i^\theta(\cdot\mid z)
-
\mathbf 1\{z'_i=\mathrm{MASK}\}\widehat p_i^\theta(\cdot\mid z')
\right\|_1
\le
L_{\mathrm{score}}d_H(z,z'),
\]
then, under the per-cell Poisson proposal coupling, the expected candidate-event switching obeys $\mathbb E_{\omega_q}
\left[
|C_q(z,\omega_q)\triangle C_q(z',\omega_q)|
\right]
\le
C\rho_qL_{\mathrm{score}}d_H(z,z')$, where $C$ is a constant depending on the way of coupling. Thus our normalized event-switching assumption is a direct algorithm-level consequence of a Dobrushin-type score sensitivity bound, with the proposal probability \(\rho_q\) separating temporal sparsity from model sensitivity. 

When most coordinates are still masked, the active posterior is high-entropy and a small perturbation of the revealed context can affect many remaining masked positions. Later in generation, most tokens have already been revealed, the active posterior becomes more confident, and both the number of active positions and the per-active-position posterior variation decrease. This predicts a decreasing \(L_{\mathrm{score}}\) profile along the generation trajectory, which supports the view that the normalized posterior map becomes more stable as the context accumulates. Such dynamics is also applicable to \(L_{\star}\). We verified the relation between $L_{\mathrm{score}}$ and the mask proportion during the generation process in Appendix~\ref{exp:L_score}.  One may also refer to Appendix~\ref{app:further_assump} for further discussion about the connection between this assumption and parallel-in-token properties of $\tau$-leaping. 

\begin{definition}[\bf Propagation masses]
\label{def:block-masses}
For each fine cell, set $b_q=L_\star\rho_q$. If block $n$ contains the fine cells $\mathcal B_n$, define the block and global propagation mass by
\[
B_n=\sum\nolimits_{q\in\mathcal B_n}b_q, \quad G_d=\sum\nolimits_{q\in\mathcal Q_d}b_q=\sum\nolimits_{n=1}^N B_n.
\]

\end{definition}
 The quantity \(b_q\) measures how many local proposal events can change per unit Hamming perturbation, $B_n$ controls the Picard difficulty of block $n$, and $G_d$ controls the total propagation mass over the whole sampling interval. In particular, $B_n$ is relatively small (\textbf{usually $\boldsymbol{\leq 1}$}) because it counts realized proposal-event switching rather than all possible score changes: a perturbation affects the Picard error only when it changes a valid proposal event, whose probability is controlled by the one-proposal mass \(\rho_q\). We also expect \(B_n\) to increase as sampling approaches the clean endpoint. In masking diffusion, the reverse unmasking scalar rate grows near the endpoint, which increases \(\lambda_q=\Delta_q a_q\) and hence the proposal probability \(\rho_q=\lambda_q e^{-\lambda_q}\). Therefore the actual propagation mass \(B_n\) can grow in the final blocks, which is precisely the regime controlled by early stopping. We also verified such dynamics and scale of \(B_n\) under text generation tasks in Appendix~\ref{exp:B_n_test}.

\begin{remark}
Although the unweighted \(L_{\mathrm{score}}\) might be numerically large in the early high-noise region (around 40 to 50 in text generation tasks, see Appendix~\ref{exp:L_score}), the tiny reverse proposal probability based on the noise schedule still ensures that the actual event-switching mass \(b_q\) and hence \(B_n\) remain small for fast convergence of Picard iteration.
\end{remark}

\subsection{Picard Convergence}
\label{sec:picard-theory}
We first present the theorem for Picard convergence.
\begin{theorem}[\bf Picard convergence]
\label{thm:picard-endpoint_main}
For any fixed time gird $[t_n,t_{n+1}]$, let $B=\sum_{r=0}^{M-1}b_r$ and initialization error $E_0=\max\limits_{0\le m\le M}\mathbb E\left[d_H\left(\widehat{y}^{(1)}_{\tau_{n,m}},\widehat{y}^{(0)}_{\tau_{n,m}}\right)\right].$
Under Assumption~\ref{ass:event-switching}, 
for every $K\ge0$,
\[
\mathbb E\left[d_H\left(\widehat{y}^{(k)}_{\tau_{n,M}},\widehat{y}^{(\infty)}_{\tau_{n,M}}\right)\right]
\le
E_0 e^B\frac{B^K}{K!}.
\]
\end{theorem}

We denote by \( \widehat{y}^{(\infty)}_{\tau_{n,M}}\) the fixed point of Picard iteration, which is exactly the sequential \(\tau\)-leaping trajectory over the same block, fine grid, and random source. Similar to the proof of Proposition 1 in \citep{shih2023parallel}, since the \(m\)-th state in each block depends only on the input states at earlier microsteps \(0,\ldots,m-1\), the Picard iteration is triangular in time. Therefore, the serial trajectory is a fixed point of the iteration, and by induction on the microstep index it is the unique fixed point. Meanwhile, since any two states can at most be different on $d$ coordinates, one may always take the initial error $E_0\le d$. Based on the experiments in Appendix~\ref{exp:B_n_test}, we verify that the propagation mass $B$ is clearly smaller than 1 except for the terminal phase of generation. 

The theorem shows that the iteration residual of discrete diffusion under Hamming distance performs exponential-factorial contraction compared with typical exponential contraction in continuous diffusion~\citep{chen2024accelerating}. 
Please refer to Appendix~\ref{appd:picard_converge} for the proof.

\begin{remark}
The exponential-factorial contraction in Theorem~\ref{thm:picard-endpoint_main} is not exclusive to the discrete diffusion. It is actually the discrete analogue of the classical Volterra expansion for Picard iterations in continuous case. Please refer to Appendix~\ref{appd:con_dis} for details.
\end{remark}

\begin{corollary}[\bf Uniform-block Picard NFE]
\label{cor:picard-nfe_main}
Assume $t_n-t_{n+1}=O(1)$ for any $n \in [N]$ and let $B_{\max}=\max_n B_n$.
Under Assumption~\ref{ass:event-switching},  Then the choice
\[
K_p=\left\lceil
\max\left\{2eB_{\max},\frac{1}{\log 2}\log\frac{Nd e^{B_{\max}}}{\varepsilon}\right\}
\right\rceil =O\left(\log (d\varepsilon^{-1})\right)
\]
ensures that $\TV(\Law(Y_N^{(K_p)}),\Law(Y_N^{(\infty)}))\leq \varepsilon$, where $\Law(\cdot)$ is the law of variables. 
\end{corollary}
Here, we prove that an $\varepsilon$-level block endpoint TV error can be achieved with $O(\log^2 d)$ total NFE under \emph{constant level physical block width}, which implies the true block-wise parallel sampling. Such result is also independent to the fine-grid partition and discretization error.
Please refer to Appendix~\ref{appd:picard_nfe} for the proof.

\subsection{Total Complexity of Algorithm~\ref{predictor}}
\label{sec:global-theory}

We now connect the Picard endpoint theorem to a sequential convergence theorem for absorbing discrete diffusion.  We take the absorbing $\tau$-leaping sampler of \citep{liangabsorb} as the serial reference that our Picard block iteration parallelizes.

Let $\mu_0$ be the clean data distribution, $\mu_\eta^\star$ the exact absorbing reverse law at early stopping time $\eta$, and $A_\eta$ the deterministic completion map used after early stopping.  Let $\nu_{\mathrm{seq}}^{\theta,\Delta}$ be the output law of the absorbing serial $\tau$-leaping sampler on the fine grid, and let $\nu_{\mathrm{pic}}^{\theta,\Delta,K_p}$ be the output law of the Picard sampler with depth $K_p$.  We define $E_{\mathrm{term}}:=\TV(\mu_0,A_\eta\#\mu_\eta^\star)$.

\paragraph{Convergence of sequential $\tau$-leaping (Theorem 2 in  \citep{liangabsorb}).}
The absorbing $\tau$-leaping theorem gives a KL bound of the form
\[
\mathrm{KL}(\mu_\eta^\star\|\nu_{\mathrm{seq}}^{\theta,\Delta})
\le E_{\mathrm{absTL}},
\]
where, up to universal and logarithmic constants,
\[
E_{\mathrm{absTL}}
:=
 d  e^{-T} \log S
 +\varepsilon_{\mathrm{score}}
 +
 \frac{dS\bigl(T+\log(M_{\mathrm{score}}\eta^{-1})\bigr)
          \bigl(T+\log \eta^{-1}\bigr)^2}
        {N_{\mathrm{fine}}}.
\]
Here $N_{\mathrm{fine}}$ is the number of serial fine steps, $\varepsilon_{\mathrm{score}}$ is the score entropy error, and $M_{\mathrm{score}}$ is the bounded-score constant.  The same result gives the early-stopping control $E_{\mathrm{term}}\lesssim d\eta$. The theorem gives a KL bound for the early-stopped serial sampler by decomposing the error into initialization, score-estimation, and time-discretization terms. 

The main technical feature of the result is that it exploits the special geometry of the absorbing rate matrix. Unlike symmetric or uniform-rate chains, the absorbing rate matrix is highly asymmetric, so standard log-Sobolev mixing arguments are not directly applicable. Instead, the analysis uses the explicit form of the absorbing transition kernel and obtains score bounds only along valid absorbing transitions, namely those pairs \((x,y)\) with \(Q(y,x)>0\). This avoids taking a uniform worst-case bound over all pairs of distinct states, which would be overly pessimistic for absorbing dynamics. 
As a result, the serial absorbing \(\tau\)-leaping sampler achieves a dimension-linear step complexity, up to logarithmic and accuracy factors, for approximating the early-stopped target distribution. 

In our analysis, this theorem is used only as the serial discretization-and-score reference; the additional term we control is the Picard parallelization error between this serial sampler and our blockwise parallel sampler.

\begin{proposition}[\bf Total Variation error with absorbing $\tau$-leaping reference]
\label{prop:total-tv_main}
Assume the absorbing serial $\tau$-leaping bound above and the Picard endpoint bound in Theorem~\ref{thm:picard-endpoint_main}.  Then
\[
\begin{aligned}
\TV(\mu_0,A_\eta\#\nu_{\mathrm{pic}}^{\theta,\Delta,K_p})
\le
E_{\mathrm{term}}+\sqrt{\frac12E_{\mathrm{absTL}}}+
\sum_{n=1}^N E_{n,0}e^{B_n}\frac{B_n^{K_p}}{K_p!}.
\end{aligned}
\]
\end{proposition}
The proposition works as a bridge connecting our Picard iteration error with the standard sequential $\tau$-leaping reference and its original truncation/discretization error. The key idea is a direct application of the triangle and Pinsker's inequality for total variation. Please refer to Appendix~\ref{appd:total_tv} for the proof. 

\begin{theorem}[\bf Complexity of Picard $\tau$-leaping]
\label{thm:global-complexity-liang_main}
Let $\varepsilon_{\mathrm{tot}}\in \mathbb{R}_+$. Under Assumption~\ref{ass:score-estimation-error},~\ref{ass:bounded-score-estimate},~\ref{ass:event-switching}, 
the Algorithm~\ref{predictor} outputs a sample with 
\[\mathrm{TV}(\mu_0,A_\eta\#\nu_{\mathrm{pic}}^{\theta,\Delta,K_p}) \leq \varepsilon_{\mathrm{tot}}\]
within $N_{\mathrm{block}}K_p
=
O\Big(
\log\frac{d \log S}{\varepsilon_{\mathrm{tot}}^2}
\cdot
\log\frac{d}{\varepsilon_{\mathrm{tot}}}
\Big)$ iterations and using space $dM=O\Big(d\frac{N_{\mathrm{fine}}}{N_{\mathrm{block}}}\Big)$
by taking hyperparamters as 
\begin{gather*}
\eta=\Theta(\varepsilon_{\mathrm{tot}}/d),
\quad
T=O\left(\log (d\varepsilon_{\mathrm{tot}}^{-2}\log S )\right),
\quad
N_{\mathrm{fine}}=\widetilde O\left(dS{\varepsilon_{\mathrm{tot}}^{-2}}\right),
\\
N_{\mathrm{block}}
=
O\left(\log (d\varepsilon_{\mathrm{tot}}^{-2}\log S )\right),
\quad
K_p
=
O\left(\log (d\varepsilon_{\mathrm{tot}}^{-1})\right).
\end{gather*}

\end{theorem}
Here, we propose the full parameter schedule for our Picard method under $\varepsilon_{\mathrm{tot}}$-level TV error, achieving ${O}(\log^2d)$ level time complexity and $\widetilde O\left({d^2S}\right)$ level space complexity. Please refer to Appendix~\ref{appd:global_comp} for proof details.  
\begin{remark}
The preceding theorem uses constant physical block width, which balances the critical path and memory.  If hardware limitations are ignored, one may instead place the entire early-stopped time horizon into a single Picard block as $N_{\mathrm{block}}=1$, which leads to even $O(\log d)$ level complexity.
\end{remark}
\begin{remark}
Here the time complexity is defined based on NFE. When considering the total workload, since the first-hitting truncation in Algorithm~\ref{appd:FHS_alg} requires $\log M$ steps of parallel prefix scan in each block, to final complexity would be $N_{\mathrm{block}}K_p\log M$. However, since $\log M=O(\log d)$, the total workload is still at $\mathrm{poly}\log(d)$ level.
\end{remark}

\section{Experiments}
\label{sec:exp}
In this section, we empirically evaluate the performance of our parallel $\tau$-leaping algorithm against sequential $\tau$-leaping and related acceleration methods.

\subsection{2D Toy Experiments}
We first conduct oracle experiments on sampling Chessboard and Circle 2D distributions under different Picard depth. The results in Table~\ref{tab:2Dtoy} show rapid convergence of the Picard Iteration with significant NFE and runtime acceleration. More detailed settings are demonstrated in Appendix ~\ref{appd:synthetic}.
\begin{table}[ht]
\centering
\captionsetup{skip=8pt}
\small
\caption{Performance comparison on synthetic data.}
\label{tab:2Dtoy}
\begin{tabular}{ll ccc ccc}
\toprule
\multirow{2}{*}{Alg.} & \multirow{2}{*}{$K_p$} & \multicolumn{3}{c}{\textbf{Chessboard}} & \multicolumn{3}{c}{\textbf{Circle}} \\
\cmidrule(lr){3-5} \cmidrule(lr){6-8}
& & Runtime (s) & KL Divergence$\downarrow$ & NFE & Runtime (s) & KL Divergence$\downarrow$ & NFE \\
\midrule
Seq. 
& / & 3.79 $\pm$ 0.08 & 0.0068 $\pm$ 0.0010 & 3200 & 4.77 $\pm$ 0.21 & 0.1160 $\pm$ 0.0094 & 3200 \\
\midrule
\multirow{6}{*}{Par.} 
& 2  & 0.12 $\pm$ 0.06 & 0.0491 $\pm$ 0.0029 & 80  & 0.31 $\pm$ 0.05 & 0.1335 $\pm$ 0.0060 & 80  \\
& 4  & 0.22 $\pm$ 0.07 & 0.0233 $\pm$ 0.0024 & 160  & 0.58 $\pm$ 0.04 & 0.1146 $\pm$ 0.0072 & 160  \\
& 6  & 0.31 $\pm$ 0.04 & 0.0093 $\pm$ 0.0017 & 240  & 0.85 $\pm$ 0.03 & 0.1118 $\pm$ 0.0073 & 240  \\
& 8  & 0.40 $\pm$ 0.06 & 0.0054 $\pm$ 0.0012 & 320  & 1.12 $\pm$ 0.04 & 0.1121 $\pm$ 0.0068 & 320  \\
& 10 & 0.53 $\pm$ 0.07 & 0.0056 $\pm$ 0.0014 & 400  & 1.37 $\pm$ 0.05 & 0.1119 $\pm$ 0.0059 & 400 \\
\bottomrule
\end{tabular} 
\end{table}

\subsection{Dimensional Scaling}

We next evaluate whether the Picard sampler exhibits favorable scaling with the sequence dimension in a controlled oracle setting.  We test whether a growing number of sequential Picard blocks is sufficient to match a serial fine-grid reference as $d$ increases. We use a block-product two-mode distribution on binary sequences. Each sample of length $d$ is partitioned into groups of size $g=8$, and each group independently follows $q_g(y)=\frac{1-\alpha}{2}\delta_{0^g}(y)+\frac{1-\alpha}{2}\delta_{1^g}(y)+\alpha 2^{-g}$ with $\alpha=0.05$.
\begin{table}[t]
\centering
\captionsetup{skip=8pt}
\caption{Oracle synthetic quality-matched scaling.}
\small
\label{tab:synthetic-scaling}
\begin{tabular}{lcccccccc}
\toprule
$d$ & Seq NFE & Blocks & $M$ & Para NFE & NFE speedup & Wall speedup & Seq KL & Para KL \\
\midrule
256 & 288 & 48 & 6 & 96 & 3.0$\times$ & 1.28$\times$ & 0.0389 & 0.0418 \\
512 & 540 & 54 & 10 & 108 & 5.0$\times$ & 2.06$\times$ & 0.0212 & 0.0227 \\
1024 & 1080 & 60 & 18 & 120 & 9.0$\times$ & 4.03$\times$ & 0.0141 & 0.0164 \\
2048 & 2048 & 64 & 32 & 128 & 16.0$\times$ & 8.73$\times$ & 0.0081 & 0.0097 \\
4096 & 4104 & 72 & 57 & 144 & 28.5$\times$ & 12.45$\times$ & 0.0045 & 0.0049 \\
\bottomrule
\end{tabular}
\end{table}

\begin{figure}[ht]
    \centering
    \begin{minipage}[t]{0.48\linewidth}
        \centering
        \includegraphics[width=\linewidth]{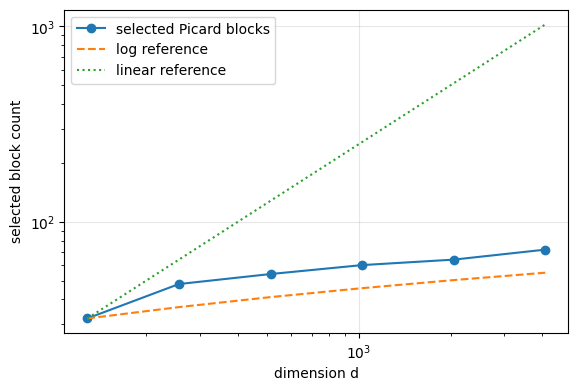}
        \centerline{\small (a) Selected Picard block count}
    \end{minipage}
    \hfill
    \begin{minipage}[t]{0.48\linewidth}
        \centering
        \includegraphics[width=\linewidth]{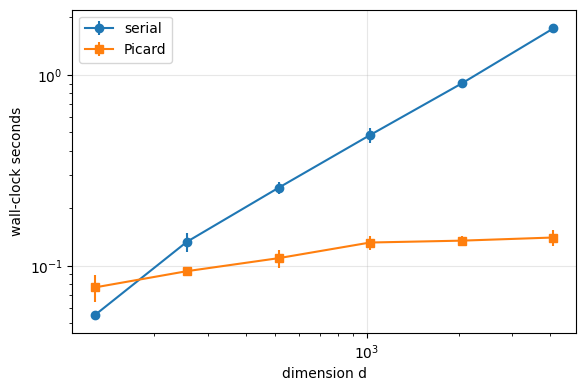}
        \centerline{\small (b) Serial vs. Picard wall-clock}
    \end{minipage}
    \caption{Visualization for the scaling experiment.}
    \label{fig:synthetic_scaling_small}
\end{figure}
Table~\ref{tab:synthetic-scaling} and Figure~\ref{fig:synthetic_scaling_small} show the quality-matched scaling behavior. As \(d\) increases, the serial reference requires a linearly growing fine grid, while the selected Picard schedule uses a much more slowly growing number of blocks, closely following the logarithmic reference. With fixed \(K_p=2\), this leads to a rapidly increasing critical-path NFE speedup, and the wall-clock measurements follow the same qualitative trend while maintaining near-serial KL quality. More visualization is in Appendix~\ref{appd:synthetic}.

\subsection{Image Generation}

The experiment utilized a MaskGiT-based score model \citep{besnier2023pytorch,chang2022maskgit} pretrained on ImageNet \citep{deng2009imagenet}. We compared the performance of both sequential and parallel $\tau$-leaping when generating $256\times256$ resolution images and evaluated the The Fréchet Inception Distance (FID) score based on 20k samples. We use classifier-free guidance with guidance scale $w$ = 3 to match the settings of baseline methods and verify the ability of combining our algorithm with inference-time control methods. The result in Table~\ref{tab:imagenet_data} shows that our parallel method achieves equivalent quality with half of the NFE. Check Appendix~\ref{appd:realworld} for more generated sample images.

\subsection{Text Generation}
\begin{wraptable}{r}{0.48\textwidth}
\centering
\captionsetup{skip=8pt}
\small
\caption{Comparison in image generation.}
\label{tab:imagenet_data}
\begin{tabular}{lccc}
\toprule
\textbf{Method}   & \textbf{NFE}  & \textbf{FID}$\downarrow$   \\ \midrule
FHS & 64  & 14.31 \\ 
Parallel Decoding & 64  & 13.07 \\ 
$\tau$-leaping & 64  & 7.56 \\ 
$\theta$-Trapezoidal & 64  & 6.59 \\
Parallel (Ours)   & 32  & 6.82 \\ 
\bottomrule
\end{tabular}
\end{wraptable}
The experiment utilized an RADD-based score model \citep{ou2024your} pretrained on the OpenWebText dataset \citep{Gokaslan2019OpenWeb} which has GPT-2-level text generation capabilities \citep{radford2019language}. We compared the performance of both sequential and parallel $\tau$-leaping when generating 1024-token texts with vocabulary $S = 50258$, and evaluated the average generative perplexity score with 1024 samples. The results in Table~\ref{tab:generative_perplexity_gpt2_large} show that our parallel method achieves better perplexity with fixed NFE, and actually achieve the same perplexity level with only half of the NFE compared with other sequential and high-order solvers. Please refer to Appendix~\ref{appd:realworld} for more comparison results with DDPD \citep{liu2024think}. 

\begin{table}[t]
\centering
\captionsetup{skip=8pt}
\small
\caption{Generative perplexity of texts generated by different sampling algorithms on GPT-2 large.}
\label{tab:generative_perplexity_gpt2_large}
\begin{tabular}{lccccccc}
\toprule
\textbf{Method}  & NFE = 32 & NFE = 64 & NFE = 128 & NFE = 256 & NFE = 512 & NFE = 1024 \\
\midrule
FHS &  $ 188.653$ & $ 140.420$ & $ 124.335$ & $ 111.959$ & $ 113.854$ & $ 110.946$ \\
Tweedie $\tau$-leaping & $ 160.466$ & $ 108.431$ & $ 83.922$ & $ 69.745$ & $ 53.922$ & $ 43.451$ \\
$\tau$-leaping & $ 94.918$ & $ 67.544$ & $ 52.384$ & $ 41.880$ & $ 35.498$ & $ 30.762$ \\
$\theta$-Trapezoidal & $ 87.895$ & $ 68.587$ & $ 50.226$ & $ 39.119$ & $ 32.886$ & $ 25.309$ \\
Parallel (Ours) & $ \boldsymbol{82.439}$ & $ \boldsymbol{50.858}$ & $ \boldsymbol{40.701}$ & $ \boldsymbol{32.510}$ & $ \boldsymbol{25.876}$ & $ \boldsymbol{22.099}$ \\
\bottomrule
\end{tabular}
\end{table}

\subsection{Runtime Acceleration}
\label{sec:runtime-acceleration}

In addition to quality metrics such as FID and PPL, we measure the actual sampling wall-clock time of our Picard sampler against the time-sequential $\tau$-leaping baseline. 
For a fair comparison, we match the quality-oriented sampling configuration and report pure sampling time, excluding the final image decoding overhead in the image experiment. 
The speedup is computed as the ratio between the serial runtime and the Picard runtime. The Picard parameters in both two tasks are set as $M=4,\ K_p=2$.

\begin{table}[ht]
\centering
\captionsetup{skip=8pt}
\caption{Runtime comparison between sequential $\tau$-leaping and Picard sampler. }
\label{tab:runtime-speedup}
\small
\setlength{\tabcolsep}{4.5pt}
\begin{tabular}{lcccccc}
\toprule
Task & Device & Serial NFE & Picard Block  & Picard NFE & Time (serial / Picard) & Speedup \\
\midrule
Text 
& H100 
& 128 
& $ N=32$ 
& 64 
& $1.652$s / $1.139$s 
& $1.45\times$ \\
Image 
& RTX 4090 
& 64 
& $ N=16$ 
& 32 
& $1.603$s / $0.860$s 
& $1.86\times$ \\
\bottomrule
\end{tabular}
\end{table}

Table~\ref{tab:runtime-speedup} indicates that the NFE reduction translates into practical single-GPU acceleration, although the wall-clock gain is smaller than the ideal NFE ratio due to relatively heavy memory traffic. Based on the research in continuous diffusion~\citep{shih2023parallel}, we expect better results in multi-GPU sampling.

\section{Discussion and Conclusion}

In this work, we proposed a parallel-in-time $\tau$-leaping sampler for absorbing discrete diffusion based on Picard iteration, achieving \(O(\log d)\) time complexity with TV-based error analysis and substantial improvement of NFE and runtime speed in experiments. Future work includes solving our method limitation by developing a sharper KL-level analysis of the Picard parallelization error and designing more memory-efficient variants, such as sliding-window Picard updates or adaptive block partitioning with multi-GPU experiments. 
Combining parallel-in-time sampling with improved serial solvers or sparse vocabulary computation may further improve scalability on large-scale discrete generation.

\begin{ack}
    Wei Huang was supported by JSPS KAKENHI (24K20848) and JST BOOST (JPMJBY24G6). Yu Yao was supported by JST SPRING (JPMJSP2108), and JST ASPIRE-MWI (JPMJAP2405). Huanjian Zhou was supported by Next Generation Artificial Intelligence Research Center, The University of Tokyo.
\end{ack}

\newpage

\bibliographystyle{alpha}
\bibliography{example_paper}

@article{gillespie2001approximate,
  title={Approximate accelerated stochastic simulation of chemically reacting systems},
  author={Gillespie, Daniel T},
  journal={The Journal of chemical physics},
  volume={115},
  number={4},
  pages={1716--1733},
  year={2001},
  publisher={AIP Publishing}
}

@article{liu2025parallelize,
  title={Parallelize single-site dynamics up to dobrushin criterion},
  author={Liu, Hongyang and Yin, Yitong},
  journal={Journal of the ACM},
  volume={72},
  number={1},
  pages={1--33},
  year={2025},
  publisher={ACM New York, NY}
}

@inproceedings{zhouparallel,
  title={Parallel Simulation for Log-concave Sampling and Score-based Diffusion Models},
  author={Zhou, Huanjian and Sugiyama, Masashi},
  booktitle={Forty-second International Conference on Machine Learning},
    year = {2025}
}

@article{chen2024accelerating,
  title={Accelerating diffusion models with parallel sampling: Inference at sub-linear time complexity},
  author={Chen, Haoxuan and Ren, Yinuo and Ying, Lexing and Rotskoff, Grant},
  journal={Advances in Neural Information Processing Systems},
  volume={37},
  pages={133661--133709},
  year={2024}
}

@article{goldreich2008computational,
  title={Computational complexity: a conceptual perspective},
  author={Goldreich, Oded},
  journal={ACM Sigact News},
  year={2008},
  publisher={ACM New York, NY, USA}
}

@article{chen2024convergence,
  title={Convergence analysis of discrete diffusion model: Exact implementation through uniformization},
  author={Chen, Hongrui and Ying, Lexing},
  journal={arXiv preprint arXiv:2402.08095},
  year={2024}
}

@inproceedings{cohen2023streaming,
  title={{Streaming Euclidean $k$-median and $k$-means with $o(\log n)$ Space}},
  author={Cohen-Addad, Vincent and Woodruff, David P and Zhou, Samson},
  booktitle={2023 IEEE 64th Annual Symposium on Foundations of Computer Science (FOCS)},
  year={2023},
  organization={IEEE}
}

@article{song2020denoising,
  title={Denoising diffusion implicit models},
  author={Song, Jiaming and Meng, Chenlin and Ermon, Stefano},
  journal={arXiv preprint arXiv:2010.02502},
  year={2020}
}

@article{anari2024fast,
  title={Fast parallel sampling under isoperimetry},
  author={Anari, Nima and Chewi, Sinho and Vuong, Thuy-Duong},
  journal={Proceedings of Thirty Seventh Conference on Learning Theory},
  pages = 	 {161--185},
  year = 	 {2024},
}

@book{kelly2011reversibility,
  title={Reversibility and stochastic networks},
  author={Kelly, Frank P},
  year={2011},
  publisher={Cambridge University Press}
}

@inproceedings{rendiscrete,
  title={How Discrete and Continuous Diffusion Meet: Comprehensive Analysis of Discrete Diffusion Models via a Stochastic Integral Framework},
  author={Ren, Yinuo and Chen, Haoxuan and Rotskoff, Grant M and Ying, Lexing},
  booktitle={The Thirteenth International Conference on Learning Representations},
  year={2025},
}

@article{campbell2022continuous,
  title={A continuous time framework for discrete denoising models},
  author={Campbell, Andrew and Benton, Joe and De Bortoli, Valentin and Rainforth, Thomas and Deligiannidis, George and Doucet, Arnaud},
  journal={Advances in Neural Information Processing Systems},
  volume={35},
  pages={28266--28279},
  year={2022}
}

@article{liang2025discrete,
  title={Discrete Diffusion Models: Novel Analysis and New Sampler Guarantees},
  author={Liang, Yuchen and Liang, Yingbin and Lai, Lifeng and Shroff, Ness},
  journal={arXiv preprint arXiv:2509.16756},
  year={2025}
}

@article{austin2021structured,
  title={Structured denoising diffusion models in discrete state-spaces},
  author={Austin, Jacob and Johnson, Daniel D and Ho, Jonathan and Tarlow, Daniel and Van Den Berg, Rianne},
  journal={Advances in neural information processing systems},
  volume={34},
  pages={17981--17993},
  year={2021}
}

@article{gruver2023protein,
  title={Protein design with guided discrete diffusion},
  author={Gruver, Nate and Stanton, Samuel and Frey, Nathan and Rudner, Tim GJ and Hotzel, Isidro and Lafrance-Vanasse, Julien and Rajpal, Arvind and Cho, Kyunghyun and Wilson, Andrew G},
  journal={Advances in neural information processing systems},
  volume={36},
  pages={12489--12517},
  year={2023}
}

@article{sarkar2024designing,
  title={Designing DNA With Tunable Regulatory Activity Using Score-Entropy Discrete Diffusion},
  author={Sarkar, Anirban and Kang, Yijie and Somia, Nirali and Mantilla, Pablo and Zhou, Jessica Lu and Nagai, Masayuki and Tang, Ziqi and Zhao, Chris and Koo, Peter},
  journal={bioRxiv},
  pages={2024--05},
  year={2024},
  publisher={Cold Spring Harbor Laboratory}
}

@article{zheng2023reparameterized,
  title={A reparameterized discrete diffusion model for text generation},
  author={Zheng, Lin and Yuan, Jianbo and Yu, Lei and Kong, Lingpeng},
  journal={arXiv preprint arXiv:2302.05737},
  year={2023}
}

@article{yang2023diffsound,
  title={Diffsound: Discrete diffusion model for text-to-sound generation},
  author={Yang, Dongchao and Yu, Jianwei and Wang, Helin and Wang, Wen and Weng, Chao and Zou, Yuexian and Yu, Dong},
  journal={IEEE/ACM Transactions on Audio, Speech, and Language Processing},
  volume={31},
  pages={1720--1733},
  year={2023},
  publisher={IEEE}
}

@inproceedings{lezama2022discrete,
  title={Discrete predictor-corrector diffusion models for image synthesis},
  author={Lezama, Jose and Salimans, Tim and Jiang, Lu and Chang, Huiwen and Ho, Jonathan and Essa, Irfan},
  booktitle={The Eleventh International Conference on Learning Representations},
  year={2022}
}

@article{wu2025fast,
  title={Fast-dllm: Training-free acceleration of diffusion llm by enabling kv cache and parallel decoding},
  author={Wu, Chengyue and Zhang, Hao and Xue, Shuchen and Liu, Zhijian and Diao, Shizhe and Zhu, Ligeng and Luo, Ping and Han, Song and Xie, Enze},
  journal={arXiv preprint arXiv:2505.22618},
  year={2025}
}

@article{xie2025variational,
  title={Variational Autoencoding Discrete Diffusion with Enhanced Dimensional Correlations Modeling},
  author={Xie, Tianyu and Xue, Shuchen and Feng, Zijin and Hu, Tianyang and Sun, Jiacheng and Li, Zhenguo and Zhang, Cheng},
  journal={arXiv preprint arXiv:2505.17384},
  year={2025}
}

@article{shih2023parallel,
  title={Parallel sampling of diffusion models},
  author={Shih, Andy and Belkhale, Suneel and Ermon, Stefano and Sadigh, Dorsa and Anari, Nima},
  journal={Advances in Neural Information Processing Systems},
  volume={36},
  pages={4263--4276},
  year={2023}
}

@article{chen2024fast,
  title={Fast sampling via discrete non-markov diffusion models with predetermined transition time},
  author={Chen, Zixiang and Yuan, Huizhuo and Li, Yongqian and Kou, Yiwen and Zhang, Junkai and Gu, Quanquan},
  journal={Advances in Neural Information Processing Systems},
  volume={37},
  pages={106870--106905},
  year={2024}
}

@article{sahoo2025diffusion,
  title={The diffusion duality},
  author={Sahoo, Subham Sekhar and Deschenaux, Justin and Gokaslan, Aaron and Wang, Guanghan and Chiu, Justin and Kuleshov, Volodymyr},
  journal={arXiv preprint arXiv:2506.10892},
  year={2025}
}

@article{ou2024your,
  title={Your absorbing discrete diffusion secretly models the conditional distributions of clean data},
  author={Ou, Jingyang and Nie, Shen and Xue, Kaiwen and Zhu, Fengqi and Sun, Jiacheng and Li, Zhenguo and Li, Chongxuan},
  journal={arXiv preprint arXiv:2406.03736},
  year={2024}
}

@article{zhang2025cosine,
  title={The Cosine Schedule is Fisher-Rao-Optimal for Masked Discrete Diffusion Models},
  author={Zhang, Leo},
  journal={arXiv preprint arXiv:2508.04884},
  year={2025}
}

@article{shaul2024flow,
  title={Flow matching with general discrete paths: A kinetic-optimal perspective},
  author={Shaul, Neta and Gat, Itai and Havasi, Marton and Severo, Daniel and Sriram, Anuroop and Holderrieth, Peter and Karrer, Brian and Lipman, Yaron and Chen, Ricky TQ},
  journal={arXiv preprint arXiv:2412.03487},
  year={2024}
}

@article{zhu2025di,
  title={DiMO: Distilling Masked Diffusion Models into One-step Generator},
  author={Zhu, Yuanzhi and Wang, Xi and Lathuiliere, Stephane and Kalogeiton, Vicky},
  journal={arXiv preprint arXiv:2503.15457},
  year={2025}
}

@article{yu2025parallelized,
  title={Parallelized midpoint randomization for langevin monte carlo},
  author={Yu, Lu and Dalalyan, Arnak},
  journal={Stochastic Processes and their Applications},
  pages={104764},
  year={2025},
  publisher={Elsevier}
}

@article{gupta2024faster,
  title={Faster diffusion-based sampling with randomized midpoints: Sequential and parallel},
  author={Gupta, Shivam and Cai, Linda and Chen, Sitan},
  journal={arXiv e-prints},
  pages={arXiv--2406},
  year={2024}
}

@article{zhao2024unified,
  title={Unified discrete diffusion for categorical data},
  author={Zhao, Lingxiao and Ding, Xueying and Yu, Lijun and Akoglu, Leman},
  journal={arXiv preprint arXiv:2402.03701},
  year={2024}
}

@article{chao2025beyond,
  title={Beyond Masked and Unmasked: Discrete Diffusion Models via Partial Masking},
  author={Chao, Chen-Hao and Sun, Wei-Fang and Liang, Hanwen and Lee, Chun-Yi and Krishnan, Rahul G},
  journal={arXiv preprint arXiv:2505.18495},
  year={2025}
}

@article{vignac2022digres,
  title={Digress: Discrete denoising diffusion for graph generation},
  author={Vignac, Clement and Krawczuk, Igor and Siraudin, Antoine and Wang, Bohan and Cevher, Volkan and Frossard, Pascal},
  journal={arXiv preprint arXiv:2209.14734},
  year={2022}
}

@inproceedings{gu2022vector,
  title={Vector quantized diffusion model for text-to-image synthesis},
  author={Gu, Shuyang and Chen, Dong and Bao, Jianmin and Wen, Fang and Zhang, Bo and Chen, Dongdong and Yuan, Lu and Guo, Baining},
  booktitle={Proceedings of the IEEE/CVF conference on computer vision and pattern recognition},
  pages={10696--10706},
  year={2022}
}

@article{wang2024fine,
  title={Fine-tuning discrete diffusion models via reward optimization with applications to dna and protein design},
  author={Wang, Chenyu and Uehara, Masatoshi and He, Yichun and Wang, Amy and Biancalani, Tommaso and Lal, Avantika and Jaakkola, Tommi and Levine, Sergey and Wang, Hanchen and Regev, Aviv},
  journal={arXiv preprint arXiv:2410.13643},
  year={2024}
}

@inproceedings{chang2022maskgit,
  title={Maskgit: Masked generative image transformer},
  author={Chang, Huiwen and Zhang, Han and Jiang, Lu and Liu, Ce and Freeman, William T},
  booktitle={Proceedings of the IEEE/CVF conference on computer vision and pattern recognition},
  pages={11315--11325},
  year={2022}
}

@article{besnier2023pytorch,
  title={A pytorch reproduction of masked generative image transformer},
  author={Besnier, Victor and Chen, Mickael},
  journal={arXiv preprint arXiv:2310.14400},
  year={2023}
}

@inproceedings{deng2009imagenet,
  title={Imagenet: A large-scale hierarchical image database},
  author={Deng, Jia and Dong, Wei and Socher, Richard and Li, Li-Jia and Li, Kai and Fei-Fei, Li},
  booktitle={2009 IEEE conference on computer vision and pattern recognition},
  pages={248--255},
  year={2009},
  organization={Ieee}
}

@article{radford2019language,
  title={Language models are unsupervised multitask learners},
  author={Radford, Alec and Wu, Jeffrey and Child, Rewon and Luan, David and Amodei, Dario and Sutskever, Ilya and others},
  journal={OpenAI blog},
  volume={1},
  number={8},
  pages={9},
  year={2019}
}

@misc{Gokaslan2019OpenWeb,  
	title={OpenWebText Corpus},
	author={Aaron Gokaslan and Vanya Cohen},
	howpublished={\url{http://Skylion007.github.io/OpenWebTextCorpus}}, 
	year={2019}
}

@article{ren2025fast,
  title={Fast solvers for discrete diffusion models: Theory and applications of high-order algorithms},
  author={Ren, Yinuo and Chen, Haoxuan and Zhu, Yuchen and Guo, Wei and Chen, Yongxin and Rotskoff, Grant M and Tao, Molei and Ying, Lexing},
  journal={arXiv preprint arXiv:2502.00234},
  year={2025}
}

@article{huang2025complexity,
  title={On the Complexity Theory of Masked Discrete Diffusion: From $\mathrm {poly}(1/\epsilon) $ to Nearly $\epsilon $-Free},
  author={Huang, Xunpeng and Lin, Yingyu and Jain, Nishant and Wang, Kaibo and Zou, Difan and Ma, Yian and Zhang, Tong},
  journal={arXiv preprint arXiv:2509.21835},
  year={2025}
}

@article{liang2026sharp,
  title={Sharp Convergence Rates for Masked Diffusion Models},
  author={Liang, Yuchen and Tan, Zhiheng and Shroff, Ness and Liang, Yingbin},
  journal={arXiv preprint arXiv:2602.22505},
  year={2026},
  url={https://arxiv.org/abs/2602.22505}
}

@article{conforti2025non,
  title={Non-Asymptotic Convergence of Discrete Diffusion Models: Masked and Random Walk dynamics},
  author={Conforti, Giovanni and Durmus, Alain and Pham, Le-Tuyet-Nhi and Raoul, Gael},
  journal={arXiv preprint arXiv:2512.00580},
  year={2025}
}

@inproceedings{liangabsorb,
  title={Absorb and Converge: Provable Convergence Guarantee for Absorbing Discrete Diffusion Models},
  author={Liang, Yuchen and Huang, Renxiang and Lai, Lifeng and Shroff, Ness and Liang, Yingbin},
  booktitle={The Thirty-ninth Annual Conference on Neural Information Processing Systems},
  year={2025}
}

@article{dmitriev2026efficient,
  title={Efficient sampling with discrete diffusion models: Sharp and adaptive guarantees},
  author={Dmitriev, Daniil and Huang, Zhihan and Wei, Yuting},
  journal={arXiv preprint arXiv:2602.15008},
  year={2026}
}

@article{zheng2024masked,
  title={Masked diffusion models are secretly time-agnostic masked models and exploit inaccurate categorical sampling},
  author={Zheng, Kaiwen and Chen, Yongxin and Mao, Hanzi and Liu, Ming-Yu and Zhu, Jun and Zhang, Qinsheng},
  journal={arXiv preprint arXiv:2409.02908},
  year={2024}
}

@article{yu2025dimple,
  title={Dimple: Discrete diffusion multimodal large language model with parallel decoding},
  author={Yu, Runpeng and Ma, Xinyin and Wang, Xinchao},
  journal={arXiv preprint arXiv:2505.16990},
  year={2025}
}

@article{ringel2026dependency,
  title={Dependency-Guided Parallel Decoding in Discrete Diffusion Language Models},
  author={Ringel, Liran and Ali, Ameen and Romano, Yaniv},
  journal={arXiv preprint arXiv:2604.02560},
  year={2026}
}

@article{liu2024think,
  title={Think while you generate: Discrete diffusion with planned denoising},
  author={Liu, Sulin and Nam, Juno and Campbell, Andrew and St{\"a}rk, Hannes and Xu, Yilun and Jaakkola, Tommi and G{\'o}mez-Bombarelli, Rafael},
  journal={arXiv preprint arXiv:2410.06264},
  year={2024}
}

@article{peng2025path,
  title={Path planning for masked diffusion model sampling},
  author={Peng, Fred Zhangzhi and Bezemek, Zachary and Patel, Sawan and Rector-Brooks, Jarrid and Yao, Sherwood and Bose, Avishek Joey and Tong, Alexander and Chatterjee, Pranam},
  journal={arXiv preprint arXiv:2502.03540},
  year={2025}
}

@article{razavi2019generating,
  title={Generating diverse high-fidelity images with vq-vae-2},
  author={Razavi, Ali and Van den Oord, Aaron and Vinyals, Oriol},
  journal={Advances in neural information processing systems},
  volume={32},
  year={2019}
}

@article{glauber1963time,
  title={Time-dependent statistics of the Ising model},
  author={Glauber, Roy J},
  journal={Journal of mathematical physics},
  volume={4},
  number={2},
  pages={294--307},
  year={1963},
  publisher={American Institute of Physics}
}

@article{kleinberg2002approximation,
  title={Approximation algorithms for classification problems with pairwise relationships: Metric labeling and Markov random fields},
  author={Kleinberg, Jon and Tardos, Eva},
  journal={Journal of the ACM (JACM)},
  volume={49},
  number={5},
  pages={616--639},
  year={2002},
  publisher={ACM New York, NY, USA}
}

@inproceedings{holenstein2007parallel,
  title={Parallel repetition: simplifications and the no-signaling case},
  author={Holenstein, Thomas},
  booktitle={Proceedings of the thirty-ninth annual ACM symposium on Theory of computing},
  pages={411--419},
  year={2007}
}

@article{bavarian2020optimality,
  title={Optimality of correlated sampling strategies},
  author={Bavarian, Mohammad and Ghazi, Badih and Haramaty, Elad and Kamath, Pritish and Rivest, Ronald L and Sudan, Madhu},
  journal={Theory of Computing},
  volume={16},
  number={1},
  pages={1--18},
  year={2020},
  publisher={Theory of Computing Exchange}
}

@article{LionsEtAl2001,
author = {Lions, Jacques-Louis and Maday, Yvon and Turinici, Gabriel},
year = {2001},
month = {01},
pages = {},
title = {A “parareal” in time discretization of PDE’s},
volume = {332},
journal = {Comptes Rendus de l’Académie des Sciences. Série I. Mathématique}
}

@article{gander2007analysis,
  title={Analysis of the parareal time-parallel time-integration method},
  author={Gander, Martin J and Vandewalle, Stefan},
  journal={SIAM Journal on Scientific Computing},
  volume={29},
  number={2},
  pages={556--578},
  year={2007},
  publisher={SIAM}
}

@article{emmett2012toward,
  title={Toward an efficient parallel in time method for partial differential equations},
  author={Emmett, Matthew and Minion, Michael},
  journal={Communications in Applied Mathematics and Computational Science},
  volume={7},
  number={1},
  pages={105--132},
  year={2012},
  publisher={Mathematical Sciences Publishers}
}

@article{gander2023unified,
  title={A unified analysis framework for iterative parallel-in-time algorithms},
  author={Gander, Martin J and Lunet, Thibaut and Ruprecht, Daniel and Speck, Robert},
  journal={SIAM Journal on Scientific Computing},
  volume={45},
  number={5},
  pages={A2275--A2303},
  year={2023},
  publisher={SIAM}
}

@inproceedings{tang2024accelerating,
  title={Accelerating parallel sampling of diffusion models},
  author={Tang, Zhiwei and Tang, Jiasheng and Luo, Hao and Wang, Fan and Chang, Tsung-Hui},
  booktitle={Forty-first International Conference on Machine Learning},
  year={2024}
}

@article{selvam2024self,
  title={Self-refining diffusion samplers: Enabling parallelization via parareal iterations},
  author={Selvam, Nikil R and Merchant, Amil and Ermon, Stefano},
  journal={Advances in Neural Information Processing Systems},
  volume={37},
  pages={5429--5453},
  year={2024}
}

@inproceedings{so2025pcm,
  title={PCM: Picard Consistency Model for Fast Parallel Sampling of Diffusion Models},
  author={So, Junhyuk and Shin, Jiwoong and Jang, Chaeyeon and Park, Eunhyeok},
  booktitle={Proceedings of the Computer Vision and Pattern Recognition Conference},
  pages={23313--23322},
  year={2025}
}

@article{lu2022dpm,
  title={Dpm-solver: A fast ode solver for diffusion probabilistic model sampling in around 10 steps},
  author={Lu, Cheng and Zhou, Yuhao and Bao, Fan and Chen, Jianfei and Li, Chongxuan and Zhu, Jun},
  journal={Advances in neural information processing systems},
  volume={35},
  pages={5775--5787},
  year={2022}
}

@inproceedings{song2023consistency,
  title={Consistency models},
  author={Song, Yang and Dhariwal, Prafulla and Chen, Mark and Sutskever, Ilya},
  booktitle={Proceedings of the 40th International Conference on Machine Learning},
  pages={32211--32252},
  year={2023}
}

@article{shen2019randomized,
  title={The randomized midpoint method for log-concave sampling},
  author={Shen, Ruoqi and Lee, Yin Tat},
  journal={Advances in Neural Information Processing Systems},
  volume={32},
  year={2019}
}

@article{anderson1965iterative,
  title={Iterative procedures for nonlinear integral equations},
  author={Anderson, Donald G},
  journal={Journal of the ACM (JACM)},
  volume={12},
  number={4},
  pages={547--560},
  year={1965},
  publisher={ACM New York, NY, USA}
}

\newpage
\appendix

\paragraph{Impact Statement}

This work primarily focuses on the algorithmic parallelization and theoretical complexity analysis of discrete diffusion models. As a fundamental algorithmic study, our core contribution lies in improving sampler convergence and hardware utilization via mathematical methods (Picard iteration), without introducing new model architectures or data biases. However, any improvement in generation efficiency objectively lowers the barrier for content creation. While the algorithm itself is neutral, faster generation capabilities, if misused, could accelerate the production of misinformation or spam. This is a general risk associated with generative models, which should be mitigated through content moderation and safe deployment strategies, rather than being a specific harm introduced by our work.

\paragraph{Compute Resources Report}

In this work, the text generation experiments were conducted on a single H100 96GB GPU, and all other experiments were conducted on a single RTX 4090 16G laptop GPU. 

\paragraph{Notation Table}

\begin{longtable}{p{0.28\linewidth}p{0.66\linewidth}}
\caption{Summary of notation used in the paper.}
\label{tab:notation}\\
\toprule
\textbf{Notation} & \textbf{Meaning} \\
\midrule
\endfirsthead

\toprule
\textbf{Notation} & \textbf{Meaning} \\
\midrule
\endhead

\midrule
\multicolumn{2}{r}{\emph{Continued on next page}}\\
\midrule
\endfoot

\bottomrule
\endlastfoot

\(d\) & Data dimension, sequence length, or number of coordinates. \\
\(S\) & Vocabulary or codebook size. \\
\(\mathcal X\)  & Generic finite CTMC state space. \\
\(\mathcal X_0=[S]^d\) & Clean state space. \\
\(\mathcal X_M=([S]\cup\{\mathrm{MASK}\})^d\) & Masked state space. \\
\(x,y,z\) & Discrete states; \(z\) usually denotes a partially masked state. \\
\(x_i,z_i\) & The \(i\)-th coordinate of \(x,z\). \\
\(z^{i\to\mathrm{MASK}}\) & State obtained by replacing coordinate \(i\) of \(z\) with \(\mathrm{MASK}\). \\
\(z^{i\leftarrow c}\) & State which replaced the masked coordinate \(i\) of \(z\) with token \(c\). \\
\(O(z)=\{i:z_i\neq\mathrm{MASK}\}\) & Observed, or unmasked, coordinate set of \(z\). \\
\(t\) & Forward noising time; \(t=0\) is the clean-data endpoint. \\
\(T\) & Maximum noising time and starting time of the reverse sampler. \\
\(\eta\) & Early stopping time near the clean endpoint. \\
\(\tau_i\) & First reveal, or first-hitting, time of coordinate \(i\). \\

\midrule
\multicolumn{2}{l}{\textbf{Masking diffusion notation}}\\
\midrule

\(\beta_t\) & Forward masking rate, or noise schedule. \\
\(\alpha_t=\exp(-\int_0^t\beta_s\,ds)\) & Probability that a coordinate remains unmasked by time \(t\). \\
\(q_t(z\mid x_0)\) & Forward marginal from clean sample \(x_0\) to masked state \(z\). \\
\(p_{\mathrm{data}}\) & Clean data distribution. \\
\(X_0\) & Clean random variable sampled from \(p_{\mathrm{data}}\). \\
\(\pi_i(c\mid z)\) & Clean-data posterior probability. \\
\(a_t=\frac{\beta_t\alpha_t}{1-\alpha_t}\) & Scalar time factor in the reverse unmasking rate. \\
\(\widehat p_i^\theta(c\mid z)\) & Learned conditional token posterior approximating \(\pi_i(c\mid z)\). \\



\midrule
\multicolumn{2}{l}{\textbf{Fine-cell, event-switching, and Picard analysis notation}}\\
\midrule

\(\mathcal Q_d\) & Set of fine-grid cells covering the early-stopped interval \([\eta,T]\). \\
\(q\in\mathcal Q_d\) & Fine cell index. \\
\(\Delta_q\) & Width of fine cell \(q\). \\
\(t_q\) & Representative time of fine cell \(q\). \\
\(a_q=a(t_q)\) & Total unmasking rate of one active masked coordinate in cell \(q\). \\
\(\lambda_q=\Delta_q a_q\) & Integrated one-coordinate proposal rate in cell \(q\). \\
\(\rho_q=\lambda_q e^{-\lambda_q}\) & One-proposal probability at one active coordinate in cell \(q\). \\
\(\omega_q\) & Randomness associated with fine cell \(q\). \\
\(C_q(z,\omega_q)\subseteq[d]\times[S]\) & Valid local proposal event set generated from state \(z\) in cell \(q\). \\
\((i,c)\in C_q(z,\omega_q)\) & Event that coordinate \(i\) is locally proposed as token \(c\). \\
\(\triangle\) & Symmetric difference of sets. \\
\(d_H(z,z')\) & Hamming distance between states \(z\) and \(z'\). \\
\(L_\star\) & Normalized event-switching sensitivity constant. \\
\(L_{\mathrm{score}}\) & Score or posterior sensitivity constant in the Dobrushin-type bound. \\
\(b_q=L_\star\rho_q\) & Propagation mass of fine cell \(q\). \\
\(\mathcal B_n\) & Set of fine cells contained in block \(n\). \\
\(B_n=\sum_{q\in\mathcal B_n}b_q\) & Propagation mass of block \(n\). \\
\(G_d=\sum_{n=1}^N B_n\) & Global propagation mass over the full early-stopped interval. \\
\(B=\sum_{r=0}^{M-1}b_r\) & Propagation mass of a generic block. \\
\(Z_m^{(k)}\) & State at microstep \(m\) after constructing the \(k\)-th Picard trajectory. \\
\(Y^{(k)}=Z_M^{(k)}\) & Block endpoint after \(k\) Picard iterations. \\
\(Y^{(\infty)}\) & Limiting, or fixed-point, block endpoint. \\
\(e_k(m)\) & Adjacent Picard residual at microstep \(m\). \\
\(E_0=\max_{0\le m\le M}e_0(m)\) & Maximum initial adjacent Picard residual within a block. \\
\(C_{<m}=(C_0,\ldots,C_{m-1})\) & Prefix of local proposal sets. \\
\(\Phi_m(C_{<m})\) & State obtained by applying first-hitting selection to the prefix \(C_{<m}\). \\
\(C_r^{(k)}=C_r(Z_r^{(k)},\omega_r)\) & Proposal event set generated in cell \(r\) from the \(k\)-th Picard input trajectory. \\
\(C_r(i)\) & Token proposal subset for coordinate \(i\) in cell \(r\). \\
\(\phi_i(C_{<m})\) & The \(i\)-th coordinate of \(\Phi_m(C_{<m})\). \\
\(\delta_i\) & Indicator that the two first-hitting outputs differ at coordinate \(i\). \\

\midrule
\multicolumn{2}{l}{\textbf{Error, distribution, and complexity notation}}\\
\midrule

\(\TV(P,Q)\) & Total variation distance between distributions \(P\) and \(Q\). \\
\(\Law(X)\) & Law, or distribution, of random variable \(X\). \\
\(\mathrm{KL}(P\|Q)\) & Kullback--Leibler divergence. \\
\(\mathcal L_{\mathrm{SE}}(\widehat s_{t_\ell})\) & Local score estimation error term at time \(t_\ell\). \\
\(\varepsilon_{\mathrm{score}}\) & Accumulated score estimation error. \\
\(M_{\mathrm{score}}\) & Bounded-score constant. \\
\(\mu_0\) & Clean data distribution. \\
\(\mu_\eta^\star\) & Exact absorbing reverse law at early stopping time \(\eta\). \\
\(A_\eta\) & Deterministic completion map applied after early stopping. \\
\(A_\eta\#P\) & Pushforward of distribution \(P\) through \(A_\eta\). \\
\(\nu_{\mathrm{seq}}^{\theta,\Delta}\) & Output law of the absorbing serial \(\tau\)-leaping sampler on the fine grid. \\
\(\nu_{\mathrm{pic}}^{\theta,\Delta,K_p}\) & Output law of the Picard sampler with depth \(K_p\). \\
\(E_{\mathrm{term}}\) & Early-stopping completion error. \\
\(E_{\mathrm{absTL}}\) & Error upper bound from the absorbing serial \(\tau\)-leaping reference. \\
\(E_{\mathrm{pic}}\) & Picard endpoint or parallelization error. \\
\(\varepsilon_{\mathrm{tot}}\) & Target total variation error. \\
\(h_0\) & Constant physical block width. \\
\(\widetilde O(\cdot)\) & Asymptotic order suppressing logarithmic factors. \\

\midrule
\multicolumn{2}{l}{\textbf{Appendix and experiment-specific notation}}\\
\midrule

\(F(s,x)\) & Vector field in the continuous Picard analogy. \\
\(L(s)\) & Time-dependent Lipschitz coefficient of \(F\). \\
\(B_T=\int_0^TL(s)\,ds\) & Total Lipschitz mass in the continuous Picard analogy. \\
\(g\) & Group size in the dimensional scaling synthetic experiment. \\
\(q_g(y)\) & Group-level synthetic target distribution. \\
\(\alpha\) & Mixture weight in the synthetic block-product experiment; distinct from \(\alpha_t\). \\
\(\delta_{0^g},\delta_{1^g}\) & Point masses at the all-zero and all-one groups. \\
\(r_{\mathrm{in}},r_{\mathrm{out}}\) & Inner and outer radii of the circle or annulus distribution. \\
\(w\) & Classifier-free guidance scale. \\
\(L_d^{\mathrm{switch}}\) & Empirical normalized event-switching sensitivity estimator. \\
\(U_q\) & Cumulative candidate-event set difference at fine cell \(q\). \\
\(D_q\) & Cumulative Hamming difference between adjacent Picard inputs at fine cell \(q\). \\
\(L_q=U_q/(\rho_qD_q)\) & Local event-switching ratio. \\

\end{longtable}

\section{Related Work}
\label{app:related}
\subsection{Discrete Diffusion models}
Masked and Uniformed discrete diffusion models have developed from a wide range of aspects.  \citep{austin2021structured} provides foundational formalisms for discrete diffusion—multi-nomial corruption and structured transition matrices as the absorbing states.  \citep{chao2025beyond} introduce intermediate token states between masked/unmasked to avoid redundant computation when sequences barely change across steps. \citep{vignac2022digres} performs discrete denoising on graphs by noising/denoising categorical node and edge types. \citep{gu2022vector} code sequences with discrete diffusion for text-to-image, improving quality and speed versus auto-regressive token decoders. \citep{gruver2023protein} introduces NOS guidance to design protein sequences directly in sequence space, demonstrating antibody optimization in vitro. \citep{wang2024fine} optimizes discrete diffusion generators with task rewards to design biomolecular sequences.

\subsection{Acceleration for Discrete Diffusion Sampling}
So far, there have been numerous studies on accelerating sampling for discrete diffusion models. \citep{chen2024fast} replaces the standard Markov chain with a non-Markov schedule to skip steps and cut the number of network calls without retraining. \citep{sahoo2025diffusion} adapts consistency distillation to discrete diffusion by constructing the duality connection between continuous and discrete diffusion. \citep{zhu2025di} proposes a token initialization strategy that injects randomness while maintaining similarity to teacher training distribution, achieving one-step distillation of masked diffusion models. \citep{zheng2023reparameterized, ou2024your} design equivalent reparameterization of discrete diffusion that yields more effective training and decoding strategies. \citep{wu2025fast} develops confidence-aware parallel decoding to accelerate multi-token sampling while maintaining accuracy.  \citep{zhang2025cosine} gives a principled choice of discretization schedule for efficient sampling. \citep{shaul2024flow} allows arbitrary discrete probability paths, giving more control to find shorter or easier trajectories with fewer steps for discrete generation. \citep{zhao2024unified} derives a simple backward denoising formula, enabling exact and accelerated sampling and unifying discrete-time and continuous-time discrete diffusion.

\subsection{Parallel-In-Time Acceleration and Picard Iteration}
As a classical method of solving differential equations, parallel-in-time sampling has also been widely applied to various fields. \citep{LionsEtAl2001} and \citep{gander2007analysis} proposed and analyzed the 'Parareal' method for solving PDEs which using a coarse propagator to rapidly propagate global information and running a fine propagator iteratively in each sub-time intervals for error correction. \citep{emmett2012toward} designed a iterative space-time multigrid-like PDE numerical method for which conducts spectral deferred correction sweeps on multiple time steps in parallel. \citep{gander2023unified} proposed a unified framework for these parallel-in-time numerical solvers.

\citep{selvam2024self} utilizes the parareal method to accelerate continuous diffusion sampling by generating rough samples at first and running parallel refinement to push the trajectories converge to serial ODE solutions. \citep{anari2024fast} designs a Picard-based parallelize Langevin Monte Carlo method for sampling log-density smooth continuous distributions following the log-Sobolev inequality. \citep{shih2023parallel} proposes the ParaDiGMS method for continuous diffusion acceleration which is compatible with various mainstream sequential solvers such as DDIM \citep{song2020denoising} and DPMSolver \citep{lu2022dpm}. \citep{chen2024accelerating} provides strict theoretical analysis about the error bound and time-complexity of Picard-based continuous diffusion, and \citep{zhouparallel} proposes a faster parallel-in-time continuous diffusion sampler by combining the Picard Iteration with diagonal time slices technique. \citep{gupta2024faster} and \citep{yu2025parallelized} utilizes both Picard Iteration and randomized midpoint method in log-concave sampling \citep{shen2019randomized} for faster parallel continuous diffusion sampling and Langevin Monte Carlo. \citep{tang2024accelerating} reframed fixed-point Picard iterations as a nonlinear root-finding problem, proposing an acceleration technique based on Anderson Acceleration \citep{anderson1965iterative}. \citep{so2025pcm} combines the idea of Picard Iteration with the consistency model \citep{song2023consistency}. It trains the model to directly predict the Picard fixed-point solution with exact convergence by the model switching.

\subsection{Parallel Discrete Sampling}
A related but distinct line of work studies parallel sampling for discrete distributions via Markov chains, especially single-site dynamics such as Glauber dynamics or Gibbs sampling \citep{glauber1963time}. \citep{liu2025parallelize} shows that such dynamics can be faithfully parallelized up to Dobrushin-type conditions by
pre-sampling the continuous-time update schedule and solving the resulting dependency structure in parallel, using correlated sampling to control the propagation of local disagreements \citep{kleinberg2002approximation, holenstein2007parallel, bavarian2020optimality}. This yields parallel simulation of the same underlying Markov chain rather than a new approximate transition rule.

This differs from discrete diffusion sampling based on $\tau$-leaping, where each discretization step already performs coordinate-parallel updates by evolving independent one-dimensional CTMCs with the score frozen at the beginning of the step. Thus the main question is not how to faithfully
parallelize an existing sequential single-site chain, but how many such score-frozen parallel steps are needed to control the discretization bias. In particular, Dobrushin-style local sensitivity conditions are useful for understanding stability of local proposals, but they do not directly reduce the serial number of $\tau$-leaping steps unless combined with additional structure controlling the dependence among simultaneously updated coordinates.

\paragraph{Remark}

Although continuous diffusion and its parallel-in-time acceleration methods have already been well-established with detailed theoretical and empirical analysis, due to the significant gap of mathematical framework between discrete and continuous diffusion such as random source, dynamics and regularity condition, it is difficult to directly transfer continuous analysis tools and methods to the discrete case. Table~\ref{tab:cont-vs-disc-diffusion-short} shows the main difference between discrete and continuous diffusion settings.

\begin{table}[ht]
\centering
\captionsetup{skip=8pt}
\caption{Continuous vs.\ discrete diffusion from the viewpoint of stochastic simulation.}
\label{tab:cont-vs-disc-diffusion-short}
\small
\renewcommand{\arraystretch}{1.1}
\begin{tabularx}{\columnwidth}{>{\raggedright\arraybackslash}p{0.23\columnwidth}
>{\raggedright\arraybackslash}X
>{\raggedright\arraybackslash}X}
\toprule
\textbf{Aspect} & \textbf{Continuous diffusion} & \textbf{Discrete diffusion} \\
\midrule
Randomness & Brownian motion & Poisson \\
State space & $\mathbb{R}^d$ & $([S]\cup [\mathrm{MASK}])^d$ \\
Dynamics & SDE / ODE & jump process \\
Score & $\nabla_x\log p_t(x)$ & jumping rate \\
Regularity condition & Lipschitz score & Dobrushin event-switching bound~(Assumption~\ref{ass:event-switching}) \\
\bottomrule
\end{tabularx}
\end{table}

\section{Further discussion on Assumption~\ref{ass:event-switching}}
\label{app:further_assump}

Assumption~\ref{ass:event-switching} and its alternative form of score sensitivity is analogous to a Dobrushin-type condition, where one controls the influence of a single coordinate on the update distribution of other coordinates. In the language of a weighted dependency graph, if $w_{j\to i}$ denotes the influence of token $j$ on the posterior/update rule of token $i$, such a condition roughly implies $\sum_{i}w_{j\to i}=O(1)$ for every token $j$. Thus each token has bounded total outgoing influence, similar to a bounded Dobrushin influence norm in the parallel MCMC sampling theory \citep{liu2025parallelize}. 

However, this is not the quantity that controls whether one can safely unmask a large batch of tokens in a single $\tau$-leaping step. Suppose the currently revealed set is $V$, and the next step simultaneously unmasks a set $U$. A coordinate-parallel $\tau$-leaping update effectively uses the product approximation
\[
q(x_U|x_V) \approx \prod_{i\in U}q(x_i|x_V),
\]
because all tokens in $U$ are sampled in parallel from the old state $x_V$. The error of this approximation depends on the residual dependence inside the batch $U$, , for example through the internal edge weight
\[
W(U)=\sum_{\substack{i,j\in U \\ i \neq j}}w_{j\to i}.
\]
A bounded single-node sensitivity condition only gives bounded weighted out-degree. It does not imply that $W(U)$ is small for a large batch $U$. Indeed, even if every node has $O(1)$ total outgoing influence, the full dependency graph may still have total edge weight $O(d)$. A large random batch can therefore contain $O(|U|)$, or even $O(d)$, total internal dependency weight.

This distinction is important for $\tau$-leaping since it is already token-parallel within each discretization step. Its remaining serial complexity comes from how aggressively one can increase the unmasking probability per step without introducing excessive bias. Local score sensitivity ensures that perturbations do not propagate too strongly from any single token, but aggressive unmasking requires stronger information-theoretic conditions \citep{dmitriev2026efficient} saying that the simultaneously unmasked tokens have small conditional total correlation given the previously revealed tokens at distribution level.

\section{Proofs of results in Section ~\ref{sec:picard-theory}}

\subsection{Proof of Theorem ~\ref{thm:picard-endpoint_main}}
\label{appd:picard_converge}
We first introduce and prove some necessary lemmas.
\begin{lemma}[Global mass bound for loglinear masking]
\label{lem:Gd-bound}
Assume $L_\star=O(1)$ and the loglinear masking rate satisfies $a(t)\asymp 1/t$ near the clean-data endpoint. Then
\[
G_d\le L_\star\sum_{q\in\mathcal Q_d}\Delta_q a_q.
\]
Moreover, for a sufficiently fine grid,
\[
\sum_{q\in\mathcal Q_d}\Delta_q a_q=O\left(\int_\eta^T a(t)\dd t\right)=O\left(\log\frac{T}{\eta}\right).
\]
If $T$ is at most logarithmic in $d$ and $\eta\asymp\varepsilon/d$ with a constant $\varepsilon$, then, up to lower-order logarithmic factors,
\[
G_d=O\left(\log\frac d\varepsilon\right).
\]
\end{lemma}

\begin{proof}
By Definition~\ref{def:fine-cell} and ~\ref{def:block-masses}, we have
\[
G_d=L_\star\sum_{q\in\mathcal Q_d}\rho_q \le L_\star\sum_{q\in\mathcal Q_d}\Delta_q a_q.
\]

When the fine grid resolves the interval $[\eta,T]$, the last sum is a Riemann sum for the time-rate integral. Since $a(t)\asymp1/t$ near the endpoint,
\[
\sum_{q\in\mathcal Q_d}\Delta_q a_q
=O\left(\int_\eta^T \frac{1}{t}\dd t\right)
=O\left(\log\frac{T}{\eta}\right).
\]
Substituting $\eta\asymp\varepsilon/d$ gives $\log(T/\eta)=\log(Td/\varepsilon)$. If $T$ grows at most logarithmically in $d$, the additional $\log T$ term is lower order compared with $\log(d/\varepsilon)$, which proves the claim.
\end{proof}

Fix a block and suppress its block index. Let the block contain $M$ fine cells indexed by $r=0,\ldots,M-1$. Let $Z_m^{(k)}$ be the state at the $m$-th microstep after the $k$-th Picard trajectory has been constructed, where $m=0,\ldots,M$. The block start is fixed, so $Z_0^{(k)}$ is the block input for all $k$. Let
\[
C_r^{(k)}=C_r(Z_r^{(k)},\omega_r)
\]
be the local proposal events generated in cell $r$ from the $k$-th Picard input trajectory. Define the adjacent Picard residual
\[
e_k(m)=\mathbb E\left[d_H\left(Z_m^{(k+1)},Z_m^{(k)}\right)\right].
\]
The endpoint is defined as $Y^{(k)}=Z_M^{(k)}$.

\begin{lemma}[First-hitting stability]
\label{lem:first-hit-stability}
Fix a block start state $x$. For a prefix of local proposal sets $C_{<m}=(C_0,\ldots,C_{m-1})$, let $\Phi_m(C_{<m})$ be the state obtained from $x$ by applying, for each coordinate, only its earliest proposed token in the prefix. Then, for any two prefixes $C_{<m}$ and $C'_{<m}$,
\[
d_H\bigl(\Phi_m(C_{<m}),\Phi_m(C'_{<m})\bigr)
\le
\sum_{r=0}^{m-1}|C_r\triangle C'_r|.
\]
\end{lemma}

\begin{proof}
For coordinate $i$, write
\[
C_r(i)=\{c\in[S]:(i,c)\in C_r\}.
\]
By Definition~\ref{def:fine-cell}, $C_r(i)$ contains at most one token. Let $\phi_i(C_{<m})$ be the $i$-th coordinate of $\Phi_m(C_{<m})$. Thus $\phi_i(C_{<m})=x_i$ if $C_r(i)=\emptyset$ for every $r<m$, and otherwise $\phi_i(C_{<m})$ is the token in the earliest nonempty $C_r(i)$.

Define
\[
\delta_i=\mathbf 1\{\phi_i(C_{<m})\neq\phi_i(C'_{<m})\}.
\]
We first prove the coordinate-wise inequality
\[
\delta_i\le \sum_{r=0}^{m-1}|C_r(i)\triangle C'_r(i)|.
\]
If the right-hand side is zero, then every summand is zero, so $C_r(i)=C'_r(i)$ for all $r<m$. Hence the two coordinate-level proposal histories are identical. They either both contain no proposal for coordinate $i$, in which case both outputs equal $x_i$, or they have the same earliest nonempty cell and the same token there. In both cases $\phi_i(C_{<m})=\phi_i(C'_{<m})$, so $\delta_i=0$. This proves the coordinate-wise inequality.

Summing over coordinates gives
\[
d_H\bigl(\Phi_m(C_{<m}),\Phi_m(C'_{<m})\bigr)
=\sum_{i=1}^d\delta_i
\le
\sum_{i=1}^d\sum_{r=0}^{m-1}|C_r(i)\triangle C'_r(i)|.
\]
Reordering the finite sums yields
\[
\sum_{i=1}^d\sum_{r=0}^{m-1}|C_r(i)\triangle C'_r(i)|
=
\sum_{r=0}^{m-1}\sum_{i=1}^d|C_r(i)\triangle C'_r(i)|.
\]
Since $C_r$ is the disjoint union of the coordinate-level sets $\{i\}\times C_r(i)$,
\[
\sum_{i=1}^d|C_r(i)\triangle C'_r(i)|=|C_r\triangle C'_r|.
\]
Substituting this identity proves the lemma.
\end{proof}

\begin{lemma}[Triangular Picard recursion]
\label{lem:triangular-recursion}
Under Assumption~\ref{ass:event-switching}, for every $m=0,\ldots,M$ and $k\ge0$,
\[
e_{k+1}(m)
\le
\sum_{r=0}^{m-1}b_r e_k(r).
\]
\end{lemma}

\begin{proof}
For $m=0$, both sides are zero because the block start is the same for all Picard iterates. Let $m\ge1$.

By construction, the $(k+2)$-nd Picard state at microstep $m$ is obtained by applying first-hitting selection to the local proposals generated from the $(k+1)$-st Picard input trajectory. Therefore,
\[
Z_m^{(k+2)}=\Phi_m(C_0^{(k+1)},\ldots,C_{m-1}^{(k+1)}).
\]
Similarly,
\[
Z_m^{(k+1)}=\Phi_m(C_0^{(k)},\ldots,C_{m-1}^{(k)}).
\]
Applying Lemma~\ref{lem:first-hit-stability} to these two prefixes gives
\[
d_H\left(Z_m^{(k+2)},Z_m^{(k+1)}\right)
\le
\sum_{r=0}^{m-1}|C_r^{(k+1)}\triangle C_r^{(k)}|.
\]
Taking expectations yields
\[
e_{k+1}(m)
\le
\sum_{r=0}^{m-1}\mathbb E\left[|C_r^{(k+1)}\triangle C_r^{(k)}|\right].
\]
For a fixed $r$, the two sets $C_r^{(k+1)}$ and $C_r^{(k)}$ are generated with the same cell randomness $\omega_r$ from inputs $Z_r^{(k+1)}$ and $Z_r^{(k)}$. Conditioning on these two inputs and using Assumption~\ref{ass:event-switching},
\[
\mathbb E\left[|C_r^{(k+1)}\triangle C_r^{(k)}|\mid Z_r^{(k+1)},Z_r^{(k)}\right]
\le
L_\star\rho_r d_H\left(Z_r^{(k+1)},Z_r^{(k)}\right).
\]
Taking expectation again and using $b_r=L_\star\rho_r$ gives
\[
\mathbb E\left[|C_r^{(k+1)}\triangle C_r^{(k)}|\right]
\le
b_r e_k(r).
\]
Substituting this bound into the previous display proves the recursion.
\end{proof}
Now we restate Theorem~\ref{thm:picard-endpoint_main} and start the main proof.
\begin{theorem}[Picard endpoint convergence]
\label{thm:picard-endpoint}
Let
\[
B=\sum_{r=0}^{M-1}b_r,
\qquad
E_0=\max_{0\le m\le M}e_0(m).
\]
Under Assumption~\ref{ass:event-switching}, the block endpoint iterates converge in expected Hamming distance to a limit $Y^{(\infty)}$, and for every $K\ge0$,
\[
\mathbb E\left[d_H\left(Y^{(K)},Y^{(\infty)}\right)\right]
\le
E_0 e^B\frac{B^K}{K!}.
\]
For block $n$, the same statement holds with $B=B_n$. One may always take $E_0\le d$.
\end{theorem}

\begin{proof}
We first prove a bound on adjacent endpoint residuals. We claim that for every $K\ge0$,
\[
e_K(M)\le E_0\frac{B^K}{K!}.
\]
For $K=0$, this is immediate from the definition of $E_0$. Now let $K\ge1$. Repeatedly applying Lemma~\ref{lem:triangular-recursion} gives
\[
e_K(M)
\le
\sum_{r_1<M}b_{r_1}e_{K-1}(r_1).
\]
Applying the recursion to $e_{K-1}(r_1)$ gives
\[
e_{K-1}(r_1)
\le
\sum_{r_2<r_1}b_{r_2}e_{K-2}(r_2).
\]
Substituting,
\[
e_K(M)
\le
\sum_{r_2<r_1<M}b_{r_1}b_{r_2}e_{K-2}(r_2).
\]
Continuing this expansion for $K$ steps yields
\[
e_K(M)
\le
\sum_{r_K<r_{K-1}<\cdots<r_1<M} b_{r_1}\cdots b_{r_K} e_0(r_K).
\]
Since $e_0(r_K)\le E_0$,
\[
e_K(M)
\le
E_0 S_K,
\qquad
S_K:=\sum_{r_K<\cdots<r_1<M} b_{r_1}\cdots b_{r_K}.
\]
It remains to bound $S_K$. Expanding $B^K$ gives
\[
B^K=\left(\sum_{r=0}^{M-1}b_r\right)^K
=\sum_{(i_1,\ldots,i_K)\in\{0,\ldots,M-1\}^K} b_{i_1}\cdots b_{i_K}.
\]
All terms are nonnegative. Hence $B^K$ is at least the sub-sum over tuples with $K$ distinct indices. For each unordered subset $\{a_1,\ldots,a_K\}$ of $K$ distinct indices, the product $\prod_{a}b_a$ appears exactly $K!$ times in the distinct-index tuple expansion, once for each permutation. The strictly decreasing chains in $S_K$ are in one-to-one correspondence with these unordered subsets. Therefore
\[
B^K\ge K!S_K,
\qquad
S_K\le \frac{B^K}{K!}.
\]
Thus
\[
e_K(M)\le E_0\frac{B^K}{K!}.
\]

Now fix $L>K$. By the triangle inequality for Hamming distance,
\[
d_H(Y^{(L)},Y^{(K)})
\le
\sum_{\ell=K}^{L-1}d_H(Y^{(\ell+1)},Y^{(\ell)}).
\]
Taking expectations and using $Y^{(\ell)}=Z_M^{(\ell)}$ gives
\[
\mathbb E[d_H(Y^{(L)},Y^{(K)})]
\le
\sum_{\ell=K}^{L-1}e_\ell(M)
\le
E_0\sum_{\ell=K}^{L-1}\frac{B^\ell}{\ell!}.
\]
The exponential series converges, so the right-hand side tends to zero as $K,L\to\infty$. Hence the endpoints are Cauchy in expected Hamming distance and have a limit, denoted $Y^{(\infty)}$.

Letting $L\to\infty$ in the previous inequality gives
\[
\mathbb E[d_H(Y^{(\infty)},Y^{(K)})]
\le
E_0\sum_{\ell=K}^{\infty}\frac{B^\ell}{\ell!}.
\]
Finally,
\[
\sum_{\ell=K}^{\infty}\frac{B^\ell}{\ell!}
=
\frac{B^K}{K!}\sum_{j=0}^{\infty}\frac{B^j K!}{(K+j)!}.
\]
Since $K!j!\le (K+j)!$, we have $K!/(K+j)!\le1/j!$. Hence
\[
\sum_{j=0}^{\infty}\frac{B^j K!}{(K+j)!}
\le
\sum_{j=0}^{\infty}\frac{B^j}{j!}
=e^B.
\]
Therefore
\[
\mathbb E[d_H(Y^{(\infty)},Y^{(K)})]
\le
E_0 e^B\frac{B^K}{K!}.
\]
This completes the proof.
\end{proof}

\subsection{Proof of Corollary~\ref{cor:picard-nfe_main}}
\label{appd:picard_nfe}

We restate Corollary~\ref{cor:picard-nfe_main} for reference and start the main proof.
\begin{corollary}[Uniform-block Picard NFE]
\label{cor:picard-nfe}
Let $H:=T-\eta$ be the early-stopped sampling horizon. Choose a constant physical block width $h_0>0$ independent of $d$ and $\varepsilon$, and partition $[\eta,T]$ into
\[
N=\left\lceil \frac{H}{h_0}\right\rceil
\]
uniform physical-time blocks, with the last block possibly shorter. Under Assumption~\ref{ass:event-switching}, given $T=O(\log(d\log S/\varepsilon))$, let $B_{\max}=\max_n B_n$. Then $B_{\max}\le G_d$, and the choice
\[
K_p=\left\lceil
\max\left\{2eB_{\max},\frac{1}{\log 2}\log\frac{Nd e^{B_{\max}}}{\varepsilon}\right\}
\right\rceil
\]
ensures that the Picard endpoint TV error is at most $\varepsilon$. Consequently,
\[
K_p=O\left(\log\frac d{\varepsilon}\right),
\qquad
\NFE_{\mathrm{Picard}}=NK_p=O\left(\log\frac d{\varepsilon}\cdot \log\frac {d\log S}{\varepsilon} \right).
\]
\end{corollary}

\begin{proof}
The role of the block count is to keep the physical width of each Picard block uniformly bounded. Since the early-stopped horizon has length $H=T-\eta$, choosing
\[
N=\left\lceil \frac{H}{h_0}\right\rceil
\]
implies that every non-final block has width $h_0$ and the final block has width at most $h_0$. Thus all block widths are $O(1)$. Since $h_0$ is a constant and $\eta\le T$,
\[
N\le \frac{T}{h_0}+1=O(T).
\]
The assumption $T=O(\log(d\log S/\varepsilon))$ therefore gives
\[
N=O\left(\log\frac {d\log S}{\varepsilon}\right).
\]
For block $n$, Theorem~\ref{thm:picard-endpoint_main} gives
\[
\mathbb E[d_H(Y_n^{(K_p)},Y_n^{(\infty)})]
\le
E_{n,0}e^{B_n}\frac{B_n^{K_p}}{K_p!}.
\]
By coupling the Picard and fixed-point block endpoints using the same block randomness,
\[
\TV(\Law(Y_n^{(K_p)}),\Law(Y_n^{(\infty)}))
\le
\mathbb E[d_H(Y_n^{(K_p)},Y_n^{(\infty)})].
\]
Using the first block where the two coupled chains disagree and then applying a union bound over blocks,
\[
E_{\mathrm{pic}}
\le
\sum_{n=1}^N E_{n,0}e^{B_n}\frac{B_n^{K_p}}{K_p!}.
\]
Since $E_{n,0}\le d$ and $B_n\le B_{\max}$,
\[
E_{\mathrm{pic}}
\le
Nd e^{B_{\max}}\frac{B_{\max}^{K_p}}{K_p!}.
\]
It is enough to make the last expression at most $\varepsilon$.

Using $K!\ge(K/e)^K$,
\[
\frac{B_{\max}^K}{K!}
\le
\left(\frac{eB_{\max}}{K}\right)^K.
\]
Thus it suffices that
\[
Nd e^{B_{\max}}\left(\frac{eB_{\max}}{K}\right)^K
\le
\varepsilon.
\]
Equivalently,
\[
\left(\frac{K}{eB_{\max}}\right)^K
\ge
\frac{Nd e^{B_{\max}}}{\varepsilon},
\]
or
\[
K\log\left(\frac{K}{eB_{\max}}\right)
\ge
\log\frac{Nd e^{B_{\max}}}{\varepsilon}.
\]
If $K\ge2eB_{\max}$, then $\log(K/(eB_{\max}))\ge\log2$. If also
\[
K\ge \frac{1}{\log2}\log\frac{Nd e^{B_{\max}}}{\varepsilon},
\]
then the previous inequality holds. This proves the displayed sufficient choice of $K_p$.

It remains to simplify the order. Since the blocks partition the fine grid,
\[
B_{\max}\le \sum_{n=1}^N B_n=G_d.
\]
By Lemma ~\ref{lem:Gd-bound},
\[
B_{\max}=O\left(\log\frac d{\varepsilon}\right),
\qquad
N=O\left(\log\frac {d\log S}{\varepsilon}\right).
\]
The first term in the maximum defining $K_p$ is therefore $O(\log(d/\varepsilon))$. For the second term,
\[
\log\frac{Nd e^{B_{\max}}}{\varepsilon}
=
\log N+
\log\frac d{\varepsilon}+B_{\max}.
\]
The term $\log N$ is lower order because $N$ is logarithmic, and $B_{\max}$ has the same logarithmic order. Hence the second term is also $O(\log(d/\varepsilon))$. Therefore
\[
K_p=O\left(\log\frac d{\varepsilon}\right).
\]
Combining this with $N=O(\log(d/\varepsilon))$ gives
\[
\NFE_{\mathrm{Picard}}=NK_p=O\left(\log\frac d{\varepsilon}\cdot \log\frac {d\log S}{\varepsilon} \right).
\]
\end{proof}

\subsection{Proof of the exponential-factorial contraction in continuous diffusion}
\label{appd:con_dis}

Consider the integral equation usually used in continuous diffusion analysis:
\[
X(t)
=
X_0+\int_0^t F(s,X(s))\,ds,
\qquad 0\le t\le T,
\]
and its Picard iteration
\[
X^{(k+1)}(t)
=
X_0+\int_0^t F(s,X^{(k)}(s))\,ds.
\]
Assume that \(F\) is Lipschitz in its state argument with a time-dependent coefficient \(L(s)\), namely
\[
\|F(s,x)-F(s,y)\|
\le
L(s)\|x-y\|.
\]
Define the adjacent Picard error
\[
e_k(t):=\|X^{(k+1)}(t)-X^{(k)}(t)\|.
\]
Then
\[
e_{k+1}(t)
\le
\int_0^t L(s)e_k(s)\,ds.
\]
Indeed, subtracting two consecutive Picard updates gives
\[
X^{(k+2)}(t)-X^{(k+1)}(t)
=
\int_0^t
\left[
F(s,X^{(k+1)}(s))-F(s,X^{(k)}(s))
\right]ds.
\]
Taking norms and applying the Lipschitz condition yields
\[
e_{k+1}(t)
=
\|X^{(k+2)}(t)-X^{(k+1)}(t)\|
\le
\int_0^t L(s)\|X^{(k+1)}(s)-X^{(k)}(s)\|\,ds
=
\int_0^t L(s)e_k(s)\,ds.
\]

Iterating this Volterra inequality gives an ordered-time expansion. 
Let
\[
B_T:=\int_0^T L(s)\,ds
\]
and assume \(e_0(t)\le E_0\) for all \(t\in[0,T]\). 
For \(K=1\),
\[
e_1(T)
\le
\int_0^T L(s_1)e_0(s_1)\,ds_1.
\]
For \(K=2\),
\[
e_2(T)
\le
\int_0^T L(s_1)
\int_0^{s_1}L(s_2)e_0(s_2)\,ds_2\,ds_1.
\]
Repeating this argument yields
\[
e_K(T)
\le
E_0
\int_{0<s_K<\cdots<s_1<T}
\prod_{\ell=1}^K L(s_\ell)\,
ds_K\cdots ds_1.
\]
The ordered simplex \(\{0<s_K<\cdots<s_1<T\}\) is one of the \(K!\) equal-ordering regions of \([0,T]^K\). Since the integrand \(\prod_{\ell=1}^KL(s_\ell)\) is symmetric in the variables \(s_1,\ldots,s_K\), we have
\[
\int_{0<s_K<\cdots<s_1<T}
\prod_{\ell=1}^K L(s_\ell)\,
ds_K\cdots ds_1
=
\frac{1}{K!}
\left(\int_0^T L(s)\,ds\right)^K
=
\frac{B_T^K}{K!}.
\]
Therefore,
\[
e_K(T)
\le
E_0\frac{B_T^K}{K!}.
\]

If \(X^\star\) denotes the fixed point of the integral equation, then the fixed-point error can be controlled by summing the adjacent errors:
\[
\|X^\star(T)-X^{(K)}(T)\|
\le
\sum_{\ell=K}^{\infty} e_\ell(T)
\le
E_0\sum_{\ell=K}^{\infty}\frac{B_T^\ell}{\ell!}.
\]
Using the elementary tail bound
\[
\sum_{\ell=K}^{\infty}\frac{B_T^\ell}{\ell!}
\le
e^{B_T}\frac{B_T^K}{K!},
\]
we obtain
\[
\|X^\star(T)-X^{(K)}(T)\|
\le
E_0e^{B_T}\frac{B_T^K}{K!}.
\]

For reverse SDEs with state-independent diffusion coefficient, the same argument applies pathwise when all Picard iterates are driven by the same Brownian path with fixed random seeds, since the stochastic integral term cancels after subtracting two consecutive iterates. Hence, the factor \(B^K/K!\) in our discrete first-hit Picard analysis should be viewed as the time-discrete version of the classical ordered-simplex factor in continuous Picard theory.

\section{Proofs of results in Section~\ref{sec:global-theory}}

\subsection{Proof of Proposition~\ref{prop:total-tv_main}}
\label{appd:total_tv}

We restate Proposition~\ref{prop:total-tv_main} for reference and start the main proof.

\begin{proposition}[Total TV error with absorbing $\tau$-leaping reference]
\label{prop:total-tv}
Assume the absorbing serial $\tau$-leaping bound above and the Picard endpoint bound in Theorem~\ref{thm:picard-endpoint_main}.  Then
\[
\begin{aligned}
\TV(\mu_0,A_\eta\#\nu_{\mathrm{pic}}^{\theta,\Delta,K_p})
\le
E_{\mathrm{term}}+\sqrt{\frac12E_{\mathrm{absTL}}}+
\sum_{n=1}^N E_{n,0}e^{B_n}\frac{B_n^{K_p}}{K_p!}.
\end{aligned}
\]
\end{proposition}
\begin{proof}
Insert the two intermediate distributions $A_\eta\#\mu_\eta^\star$ and $A_\eta\#\nu_{\mathrm{seq}}^{\theta,\Delta}$.  By the triangle inequality for total variation,
\[
\begin{aligned}
\TV(\mu_0,A_\eta\#\nu_{\mathrm{pic}}^{\theta,\Delta,K_p})
&\le
\TV(\mu_0,A_\eta\#\mu_\eta^\star) \\
&\quad+
\TV(A_\eta\#\mu_\eta^\star,A_\eta\#\nu_{\mathrm{seq}}^{\theta,\Delta}) \\
&\quad+
\TV(A_\eta\#\nu_{\mathrm{seq}}^{\theta,\Delta},A_\eta\#\nu_{\mathrm{pic}}^{\theta,\Delta,K_p}).
\end{aligned}
\]
The first term is $E_{\mathrm{term}}$.  Since deterministic maps cannot increase total variation,
\[
\TV(A_\eta\#P,A_\eta\#Q)
\le
\TV(P,Q),
\]
the second term is bounded by $\TV(\mu_\eta^\star,\nu_{\mathrm{seq}}^{\theta,\Delta})$.  Pinsker's inequality and the absorbing serial KL bound give
\[
\TV(\mu_\eta^\star,\nu_{\mathrm{seq}}^{\theta,\Delta})
\le
\sqrt{\frac12\mathrm{KL}(\mu_\eta^\star\|\nu_{\mathrm{seq}}^{\theta,\Delta})}
\le
\sqrt{\frac12E_{\mathrm{absTL}}}.
\]
For the third term, couple the serial and Picard samplers by using the same block random sources.  For any coupled random variables $X$ and $Y$,
\[
\TV(\Law(X),\Law(Y))
\le
\mathbb P(X\ne Y)
\le
\mathbb E[d_H(X,Y)].
\]
Applying this relation to the first block where the coupled endpoints differ and then using a union bound over blocks gives
\[
\TV(\nu_{\mathrm{seq}}^{\theta,\Delta},\nu_{\mathrm{pic}}^{\theta,\Delta,K_p})
\le
\sum_{n=1}^N E_{n,0}e^{B_n}\frac{B_n^{K_p}}{K_p!},
\]
by Theorem~\ref{thm:picard-endpoint_main}.  Combining the displayed inequalities proves the claim.
\end{proof}

\subsection{Proof of Theorem~\ref{thm:global-complexity-liang_main}}
\label{appd:global_comp}
We restate Theorem~\ref{thm:global-complexity-liang_main} for reference and start the main proof.

\begin{theorem}[Global complexity]
\label{thm:global-complexity-liang}
Let the target total variation error be $\varepsilon_{\mathrm{tot}}$.  Under Assumption~\ref{ass:score-estimation-error},~\ref{ass:bounded-score-estimate},~\ref{ass:event-switching}, suppose the score error is controlled at order $\varepsilon_{\mathrm{tot}}^2$, then choosing
\[
\eta=\Theta(\varepsilon_{\mathrm{tot}}/d),
\qquad
T=O\left(\log\frac{d\log S}{\varepsilon_{\mathrm{tot}}^2}\right),
\qquad
N_{\mathrm{fine}}=\widetilde O\left(\frac{dS}{\varepsilon_{\mathrm{tot}}^2}\right)
\]
controls the serial and early-stopping terms.  With constant physical block width, the number of blocks and Picard depth with
\[
N_{\mathrm{block}}
=
O\left(\log\frac{d\log S}{\varepsilon_{\mathrm{tot}}^2}\right),
\qquad
K_p
=
O\left(\log\frac{d}{\varepsilon_{\mathrm{tot}}}\right)
\]
makes the Picard term of order $\varepsilon_{\mathrm{tot}}$.  Consequently, the global time and space complexity are
\[
N_{\mathrm{block}}K_p
=
O\left(
\log\frac{d \log S}{\varepsilon_{\mathrm{tot}}^2}
\cdot
\log\frac{d}{\varepsilon_{\mathrm{tot}}}
\right),
\qquad
O(dM)=
O\left(d\frac{N_{\mathrm{fine}}}{N}\right)
=
\widetilde O\left(\frac{d^2S}{\varepsilon_{\mathrm{tot}}^2}\right)
\]
\end{theorem}
\begin{proof}
The early-stopping term is $O(d\eta)$, so $\eta=\Theta(\varepsilon_{\mathrm{tot}}/d)$ makes it $O(\varepsilon_{\mathrm{tot}})$.  The serial contribution enters the total TV bound through Pinsker, so we require
\[
E_{\mathrm{absTL}}=O(\varepsilon_{\mathrm{tot}}^2).
\]
The initialization part $d e^{-T}\log S$ is then controlled by taking
\[
T=O\left(\log\frac{d\log S}{\varepsilon_{\mathrm{tot}}^2}\right).
\]
The discretization part of $E_{\mathrm{absTL}}$ is
\[
\frac{dS\bigl(T+\log(M_{\mathrm{score}}\eta^{-1})\bigr)
          \bigl(T+\log \eta^{-1}\bigr)^2}{N_{\mathrm{fine}}}.
\]
Thus it is at most $O(\varepsilon_{\mathrm{tot}}^2)$ whenever
\[
N_{\mathrm{fine}}
=
\widetilde O\left(\frac{dS}{\varepsilon_{\mathrm{tot}}^2}\right),
\]
where logarithmic factors in $d$, $1/\eta$, $1/\varepsilon_{\mathrm{tot}}$, and $M_{\mathrm{score}}$ are hidden.

For the Picard part, choosing a constant physical block width gives $N_{\mathrm{block}}=O(T)$.  Moreover, Lemma~\ref{lem:Gd-bound} yields
\[
B_{\max}
\le
G_d
=
O\left(\log\frac d{\varepsilon_{\mathrm{tot}}}\right).
\]
The safe depth choice in Corollary~\ref{cor:picard-nfe_main} then gives
\[
K_p
=
O\left(\log\frac d{\varepsilon_{\mathrm{tot}}}\right).
\]
Multiplying this by $N_{\mathrm{block}}=O(T)$ proves the claimed NFE bound.
\end{proof}
\section{Algorithm of the first-hitting truncation}
\label{appd:FHS_alg}
Algorithm~\ref{alg:fht} provides details about how our first-hitting truncation is implemented during the Picard sampling. We treat this first-hitting mechanism as a single abstract operation in the theory. Given the local proposal events in a block prefix, the first-hitting operation returns the state obtained by applying, for each coordinate, only its earliest proposed token. This abstraction keeps the analysis focused on the only property needed for convergence: first-hitting does not amplify event-level discrepancies.

\begin{algorithm}[ht]
\caption{First-Hitting Truncation (FHT)}
\label{alg:fht}
\begin{algorithmic}[1]
\REQUIRE Block-start state $\widehat y_{t_n}$ and jump sequence
$\Delta\widehat y^{(k)}_0,\ldots,\Delta\widehat y^{(k)}_{M-1}$
\ENSURE Truncated jump sequence
$\Delta\widehat y^{(k)}_0,\ldots,\Delta\widehat y^{(k)}_{M-1}$

\STATE For all coordinates $i$ and microsteps $j$, define
$h_{j,i}\gets \mathbf 1\{\Delta\widehat y^{(k)}_{j,i}\neq 0\}$.
\STATE Compute the prefix counts
$s_{j,i}\gets\sum_{\ell=0}^{j}h_{\ell,i}$
for all $i,j$ using a parallel prefix scan over $j$.
\STATE For all $i,j$ in parallel, set $\Delta\widehat y^{(k)}_{j,i}
\gets
\Delta\widehat y^{(k)}_{j,i}\,
\mathbf 1\{\widehat y_{t_n,i}=\mathrm{MASK}\}\,
\mathbf 1\{s_{j,i}=1\}.
$
\RETURN $\Delta\widehat y^{(k)}_0,\ldots,\Delta\widehat y^{(k)}_{M-1}$
\end{algorithmic}
\end{algorithm}

\section{Extra Experiment Results and Details}
In this section, we provide more details and results of the experiments.

\subsection{Assumption Verification}
\label{exp:assumption_test}

In this section, we empirically verify Assumption~\ref{ass:event-switching}. To assess the normalized event-switching sensitivity $L_\star$ along actual generation trajectories, we compute an on-trajectory weighted constant estimator
$L^{switch}_d=(\sum_q U_q)/(\sum_q\rho_q D_q)$.
Here $U_q$ is the cumulative candidate-event set difference at fine cell $q$, $D_q$ is the cumulative Hamming difference between adjacent Picard inputs at that cell, and $\rho_q$ is the one-proposal probability. Equivalently, $L_d^{\mathrm{switch}}$ is a weighted average of the local ratios $L_q=U_q/(\rho_qD_q)$ with weights proportional to $\rho_qD_q$. Since $\rho_qD_q$ is the exposure of cell $q$ in the Picard error recursion, this estimator emphasizes the cells that contribute most to aggregate iteration error.

Empirically, as shown in Figure~\ref{fig:L_star}, $L_d^{\mathrm{switch}}$ remains between $2.31$ and $2.50$ for $d\in[128,640]$, with fitted log-log slope $-0.043$, supporting the assumption that the effective normalized sensitivity is dimension-independent along generation trajectories.

\begin{figure}[ht]
    \centering
    \includegraphics[width=0.8\linewidth]{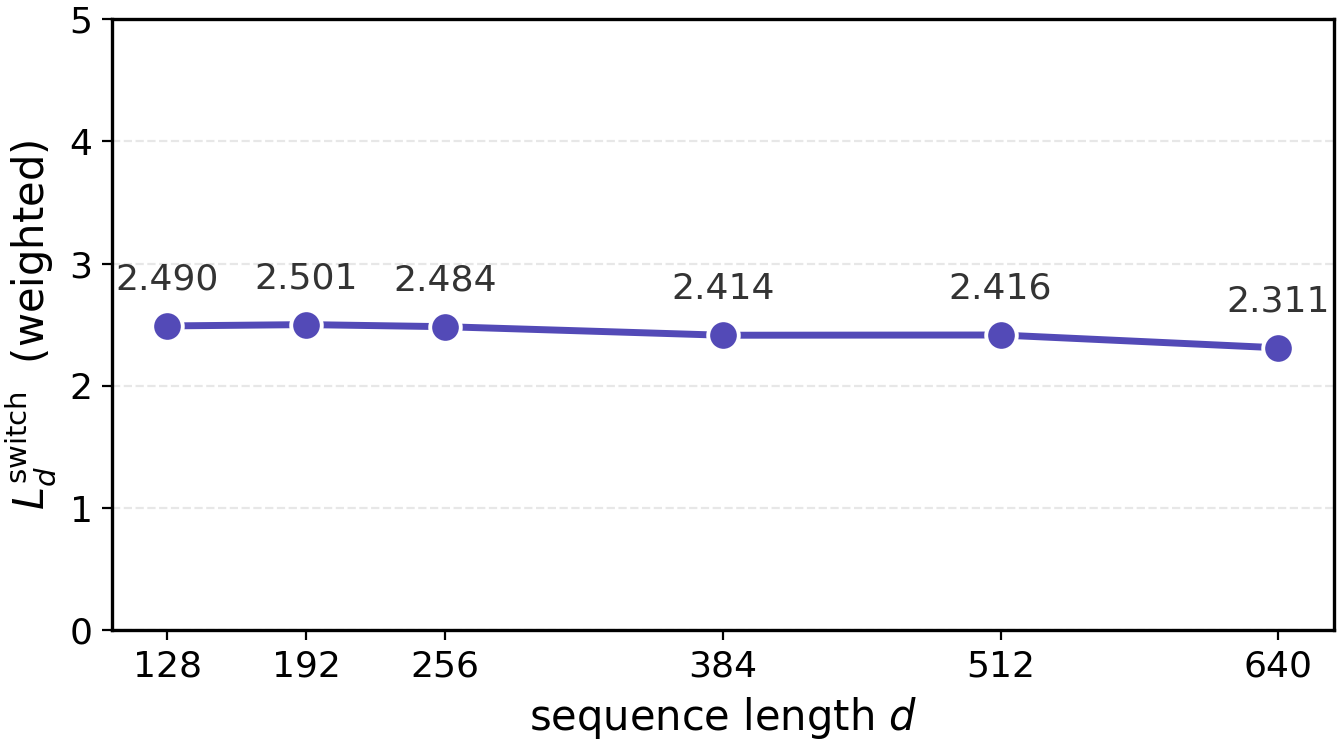}
    \caption{Visualization of $L_{\star}$.}
    \label{fig:L_star}
\end{figure}

\subsection{\texorpdfstring{$L_{\mathrm{score}}$}{Lscore} Estimation}
\label{exp:L_score}
 In this experiment, we measure the unweighted score sensitivity along RADD text generation at sequence length \(d=512\). The sampler uses \(64\) Picard blocks, \(M=4\) fine microsteps per block, and Picard depth \(K_p=2\), corresponding to \(256\) underlying fine-grid microsteps and \(128\) critical-path model evaluations. 
For a trajectory state \(z\), we construct a perturbed state \(z'\) by modifying revealed context tokens while keeping the mask pattern fixed. 
We then compare the normalized ordinary-token posteriors on the remaining MASK positions and compute
\[
    L_{\mathrm{score}}(z,z')
    =
    \frac{
    \sum_{i:z_i=\mathrm{MASK}}
    \left\|
    \widehat p_i^\theta(\cdot\mid z)
    -
    \widehat p_i^\theta(\cdot\mid z')
    \right\|_1
    }{
    d_H(z,z')
    } .
\]
As shown in Figure~\ref{fig:L_SCORE}, \(L_{\mathrm{score}}\) is largest when the context is highly incomplete and decreases as more tokens are revealed. This agrees with our expectation that when most coordinates are masked, a small context perturbation can affect many active posterior distributions; later in generation, fewer positions remain active and the posterior becomes more confident, leading to a smaller aggregate sensitivity.

\begin{figure}[ht]
    \centering
    \includegraphics[width=0.8\linewidth]{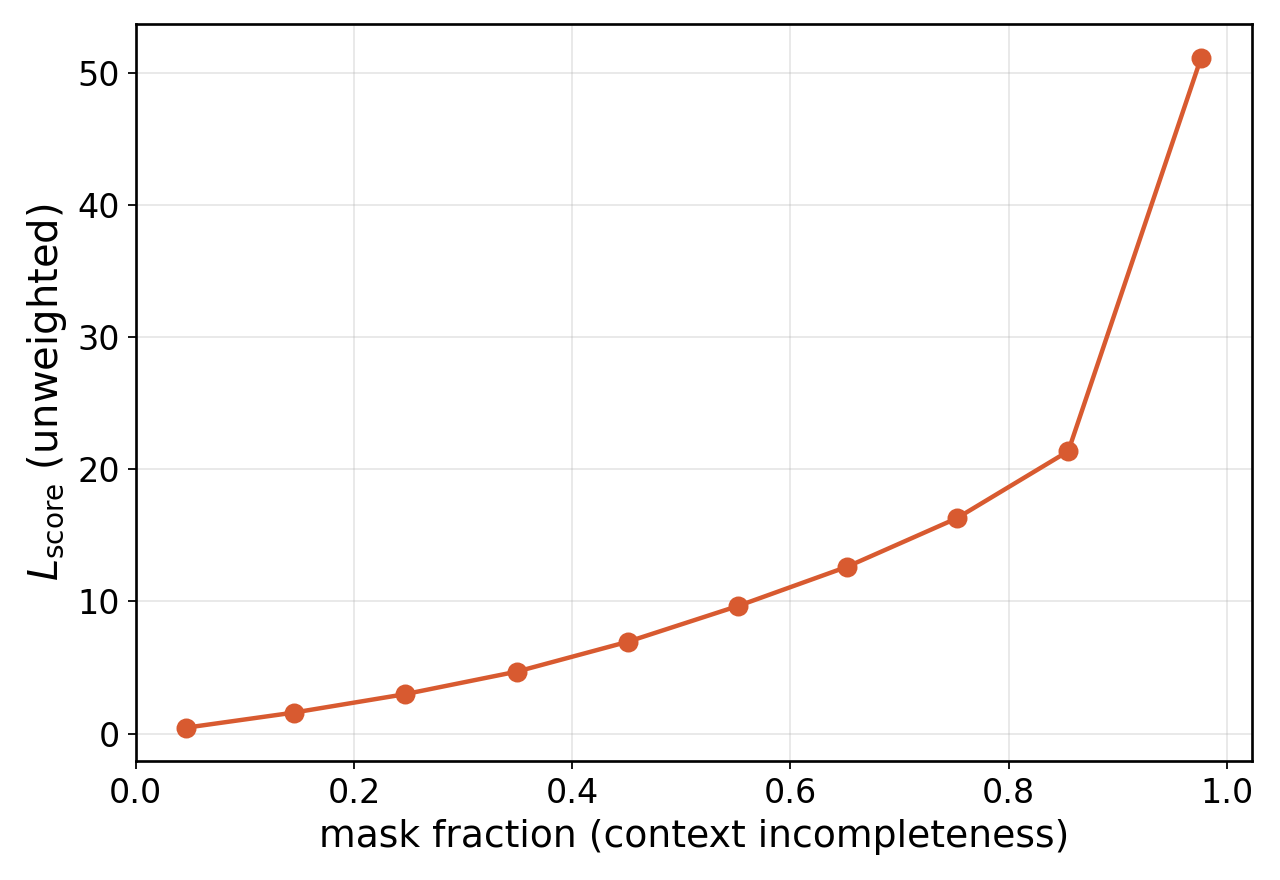}
    \caption{Visualization of $L_{\mathrm{score}}$.}
    \label{fig:L_SCORE}
\end{figure}

\subsection{Block Mass Estimation}
\label{exp:B_n_test}

To quantify the scale of \(B_n\) in Definition~\ref{def:block-masses}, we estimate it along a \(d=512\) RADD text-generation run with \(N=64\) blocks, \(M=4\) microsteps per block, \(K_p=2\), \(128\) sampling rounds, and \(16\) replay seeds per recorded state pair. As shown in Figure~\ref{fig:B_n}, the estimated block masses remain constant-scale: the mean is \(0.273\), the median is \(0.223\), the \(90\)-th percentile is \(0.528\), and the maximum is \(1.260\), with the smallest blocks close to zero. 
Thus most blocks have \(B_n<1\), while the largest values occur near the clean endpoint. 
This supports both the practical small-constant behavior assumed in the Picard contraction analysis and the need for early stopping to avoid the most singular terminal region.

\begin{figure}[ht]
    \centering
    \includegraphics[width=0.8\linewidth]{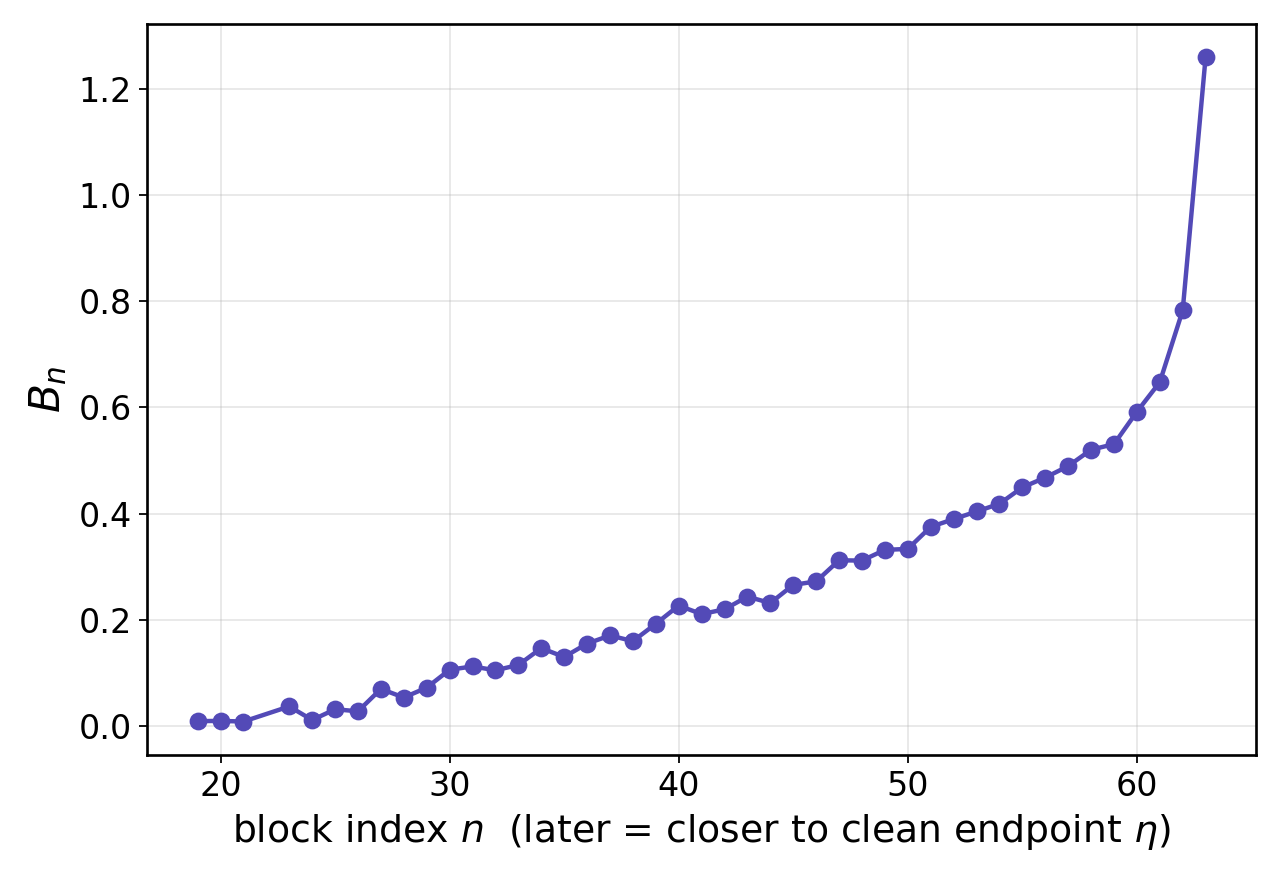}
    \caption{Visualization of $B_n$.}
    \label{fig:B_n}
\end{figure}

\subsection{Synthetic data}
In this section we provide more details about synthetic experiments. 
\label{appd:synthetic}
\subsubsection{2D Toy Model}
\paragraph{Chessboard}
The chessboard distribution has several key characteristics that make it an excellent test benchmark: (1) Discrete: The distribution is defined on a finite set of points on a two-dimensional grid; (2) Sparse: Nearly half of the grid points have a probability of exactly zero. A successful sampler must learn to restrict its generated samples strictly to the points that have non-zero probability mass (i.e., the support of the distribution); (3) Multi-modal: The probability mass is distributed across multiple, disconnected modes rather than being concentrated in a single region. The sampler is required to capture all of these distinct modes; (4) Structured: It exhibits a distinct, non-random geometric structure. This challenges the sampler's ability to reproduce the correct global pattern, rather than merely matching general statistical moments.

We conduct the experiment on a $8\times8$ chessboard distribution with varying Picard iteration depths $K_p$. We fix $N=40,M=80$ with totally 4096 samples. Runtimes and KL Divergence are averaged over 20 runs.

\begin{figure}[ht]
\begin{center}
\includegraphics[width=0.9\textwidth]{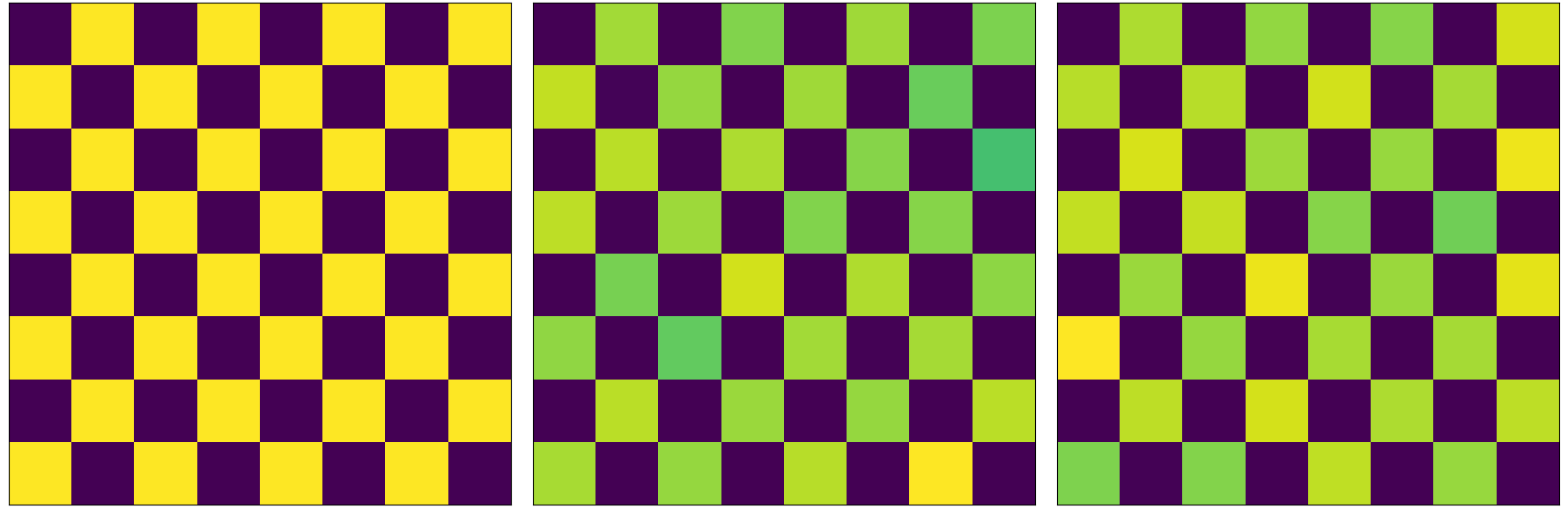}
\end{center}
\caption{Visualization of the chessboard experiments. (Left) The target distribution. (Middle) The Picard sampling result. (Right) The sequential sampling result.}
\end{figure}

\paragraph{Circle}
The ring distribution on a 2D discrete grid concentrates its entire probability mass on grid points located within an annulus defined by an inner radius $r_{in}$ and an outer radius $r_{out}$. A key characteristic is its Non-Convexity; the high-probability region encloses a central "hole" of zero probability, which leads to the distribution's support non-convex. Another defining feature is its Connectivity. Unlike the disjoint, multi-modal structure of the checkerboard distribution, the support of the ring distribution forms a single connected component, meaning any point on the ring can traverse to any other point through a series of steps to adjacent locations.

We conduct the experiment on a circle distribution at $32\times32$ 2D grid with varying Picard iteration depths $K_p$. We fix $N=40,M=80$ with totally 4096 samples. Runtime and KL Divergence are averaged over 20 runs.

\begin{figure}[ht]
\begin{center}
\includegraphics[width=0.9\textwidth]{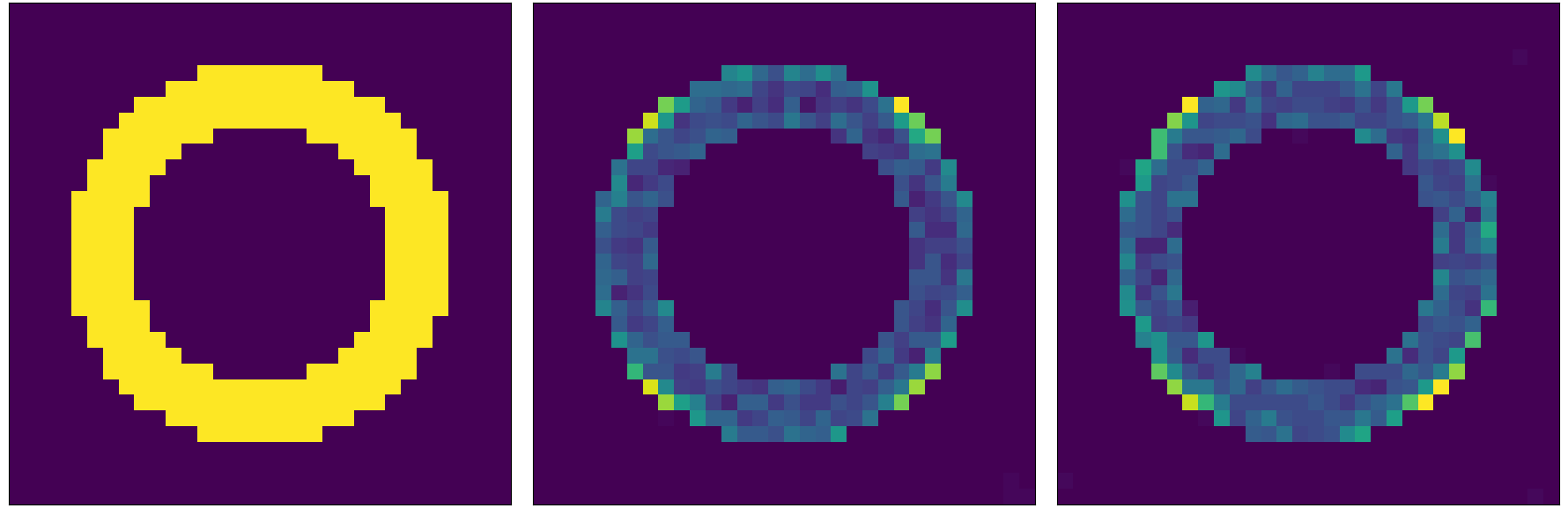}
\end{center}
\caption{Visualization of the circle experiments. (Left) The target distribution. (Middle) The Picard sampling result. (Right) The sequential sampling result.}
\end{figure}

\subsubsection{Dimensional Scaling}

\paragraph{Oracle target and metrics.}
The oracle conditional probabilities are computed exactly from the partially unmasked group state to remove score approximation error. Since each sequence contains $d/g$ independent groups, reliable group-level statistics can be obtained with a small batch size; we use only $8$ samples for every $d$ to avoid large-batch bandwidth effects. Quality is measured on the empirical group distribution.  
We report the per-coordinate group KL together with group TV and off-mode mass during schedule selection. The serial reference uses an oracle masked tau-leaping fine grid with $N_{\mathrm{fine}}(d)\approx d$. For the Picard sampler we fix $K_p=2$ and search over logarithmic block schedules, selecting the smallest block count whose sampling quality remain within prescribed margins of the serial reference. More visualization results are in Figure~\ref{fig:synthetic_scaling}.

\begin{figure}[ht]
    \centering

    \begin{minipage}[t]{0.48\linewidth}
        \centering
        \includegraphics[width=\linewidth]{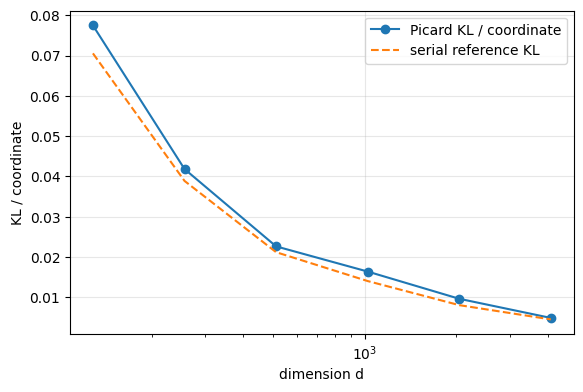}
        \centerline{\small (a) Quality under selected schedule}
    \end{minipage}
    \hfill
    \begin{minipage}[t]{0.48\linewidth}
        \centering
        \includegraphics[width=\linewidth]{figures/synthetic_block_scaling.png}
        \centerline{\small (b) Selected Picard block count}
    \end{minipage}

    \vspace{0.8em}

    \begin{minipage}[t]{0.48\linewidth}
        \centering
        \includegraphics[width=\linewidth]{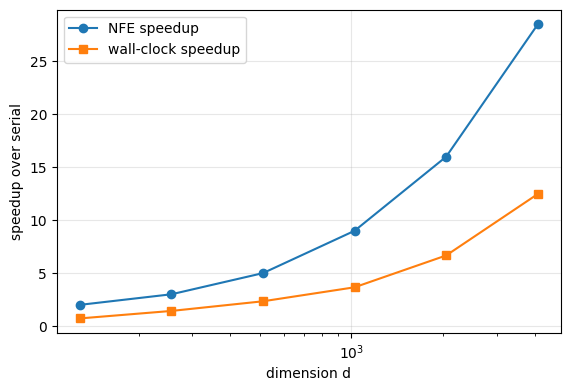}
        \centerline{\small (c) NFE and wall-clock speedup}
    \end{minipage}
    \hfill
    \begin{minipage}[t]{0.48\linewidth}
        \centering
        \includegraphics[width=\linewidth]{figures/synthetic_wallclock.png}
        \centerline{\small (d) Serial vs. Picard wall-clock}
    \end{minipage}

    \caption{
    Oracle synthetic quality-matched scaling experiment. 
    Panel (a) compares the per-coordinate group KL of the selected Picard schedule with the serial oracle reference. 
    Panel (b) shows that the selected number of Picard blocks grows much more slowly than the linear reference. 
    Panel (c) compares the critical-path NFE speedup with the measured wall-clock speedup. 
    Panel (d) reports the absolute wall-clock sampling time of the serial and Picard samplers.}
    \label{fig:synthetic_scaling}
\end{figure}

\subsection{Real-world data}
\label{appd:realworld}

In this section we provide more results and samples for text/image generation tasks. The real-data experiment code is built based on the open-source codebase of~\citep{ren2025fast}. 

Figure~\ref{fig:img_samples} demonstrates sample images generated by our parallel method. Table~\ref{tab:DDPD_res} shows the comparison results between DDPD \citep{liu2024think} and our parallel method.

Table~\ref{tab:memory} demonstrates the GPU memory cost under certain parameter settings for both image and text generation tasks. 

We would like to mention that, the MaskGIT we used in the image experiment utilizes the VQVAE \citep{razavi2019generating} to compress images into a discrete latent space with only 1024 codebook size and 16*16 sequence length, while the RADD for text generation works in a raw high-dimensional token space. Therefore, MaskGIT has much lower codebook and sequence dimensions, which can significantly decrease the active memory cost for each token sequence in the transformer. What’s more, a large portion of memory for MaskGIT is likely occupied by static model weights, which do not scale with the parallel width M. The dynamic increase from parallelization is a small fraction of the total footprint.

\begin{figure}[t]
\begin{center}
\includegraphics[width=0.8\textwidth]{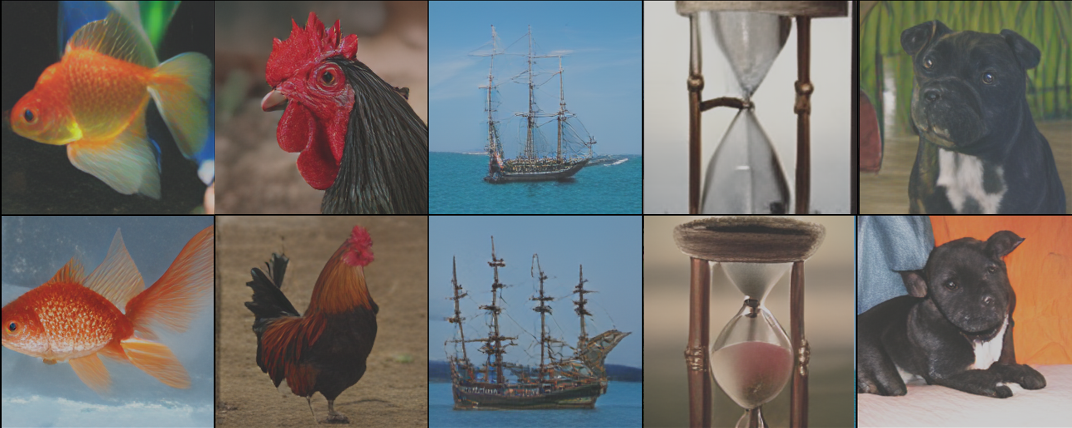}
\end{center}
\caption{Generated samples from the Imagenet experiments.}
\label{fig:img_samples}
\end{figure}

\begin{table}[htbp]
\centering
\captionsetup{skip=8pt}
\caption{Generative perplexity of texts generated by DDPD and Picard $\tau$-leaping}
\label{tab:DDPD_res}
\begin{tabular}{lcc}
\toprule
\textbf{Method} & \textbf{NFE} & \textbf{Perplexity} \\
\midrule
DDPD-small softmax & 1024 & 41.342 \\
DDPD-medium softmax & 1024 & 34.166 \\
DDPD-small sigmoid & 1024 & 30.587 \\
DDPD-medium sigmoid & 1024 & 28.025 \\
Ours & 512 & 25.876 \\
\bottomrule
\end{tabular}
\end{table}

\begin{table}[htbp]
\centering
\captionsetup{skip=8pt}
\caption{Parameter settings and GPU memory cost for image and text generation.}
\label{tab:memory}
\begin{tabular}{lccccc}
\toprule
Tasks & \textbf{Seq. Cost} &\textbf{Para. Cost} & \textbf{M}  & \textbf{N}  & \textbf{$K_p$} \\ \midrule
Text  & 6.7GB             & 11.6GB            & 8  & 32 & 2    \\ \midrule
Image & 3.0GB             & 3.7GB             & 10 & 20 & 2    \\ 
\bottomrule
\end{tabular}
\end{table}

\clearpage

\end{document}